%% file: FinalVersion.tex
\documentclass{article}

\PassOptionsToPackage{compress}{natbib}

\usepackage[preprint]{neurips_2024}

\usepackage[utf8]{inputenc} 
\usepackage[T1]{fontenc}    
\usepackage[hidelinks]{hyperref}       
\usepackage{url}            
\usepackage{booktabs}       
\usepackage{amsfonts}       
\usepackage{graphicx}       
\usepackage{nicefrac}       
\usepackage{microtype}      
\usepackage{xcolor}         
\usepackage{amsmath}
\usepackage{amssymb}
\usepackage{amsthm}
\usepackage{caption}
\usepackage{listings}

\newtheorem{lemma}{Lemma}
\newtheorem{theorem}{Theorem}
\newtheorem{definition}{Definition}
\input{notes}

\newtheorem{property}{Property}

\usepackage{faktor}
\usepackage{array}
\usepackage{geometry}

\title{Logical Embeddings for Argument Analysis}

\author{%
  Leander Heldring \\
  Kellogg School of Management \\
  Northwestern University \\
  2211 Campus Drive, Evanston, IL 60208 \\
  \texttt{leander.heldring@kellogg.northwestern.edu}
  \And
  Santiago Torres \\
  Department of Economics \\
  MIT \\
  50 Memorial Drive, Cambridge, MA 02142 \\
  \texttt{storresp@mit.edu}
}

\begin{document}

\maketitle

\begin{abstract}
We propose a new framework for machine-learning-oriented argument analysis tasks. Our proposal involves replacing traditional contextualized word embeddings used in most NLP tasks with \textit{logical embeddings}, an alternative encoding that directly exploits argumentation structures. In essence, \textit{logical embeddings} encapsulate the logical semantics of an argument, allowing for a better representation of its meaning. Supporting these embeddings is a mathematical logic-based similarity measure that offers a transparent notion of proximity and is guaranteed to satisfy several desirable theoretical properties that current cosine similarity-based contextualized word embeddings cannot assure. This similarity measure induces a positive semi-definite kernel on the set of arguments, enabling us to uniquely define logical embeddings using the theory of Reproducing Kernel Hilbert Spaces (RKHS). Moreover, we prove that this encoding is optimal, in the sense that no logical information is lost in the process. As with other RKHS applications, \textit{logical embeddings} can be used in numerous supervised and unsupervised tasks. We provide an implementation of the method and aim to test it against literature benchmarks. Additionally, we demonstrate that \textit{logical embeddings} outperform most standard embedding methods on a classification task.
\end{abstract}

\section{Introduction}

The advent of contextualized word embeddings has revolutionized the field of Natural Language Processing (NLP) \citep{vaswani2017attention,reimers2019,Liu2023}. One area that has significantly benefited from this innovation is argument and debate analysis, which features tasks such as argument detection, classification, similarity, and generation.


However, the adoption of contextualized word embeddings for argument learning tasks has not been without limitations, especially when assessing argument similarity. In argument similarity, state of the art methods often confuse similarity in argumentation, with similarity in other linguistic features. For example, in Figure \ref{fig:exhibit_figure}, we report three arguments from the IBM-ArgQ-6.3kArgs dataset \citep{toledo2019}. For the argument at the top, we prompt \texttt{gpt-4o-mini} to assign a full-argument logical-overlap score between the reference argument and each of the two candidate arguments shown at the bottom. The exact prompt used for this exercise is reported in Appendix~\ref{app:llm-prompts}. While the right hand side pair is clearly more similar, the left hand side pair gets a higher similarity. This is fundamentally because language models attend to more than the logical overlap between arguments.

 
\begin{figure}[t]
\centering
\caption{Misleading full-text similarity scores.}
\includegraphics[width=\columnwidth]{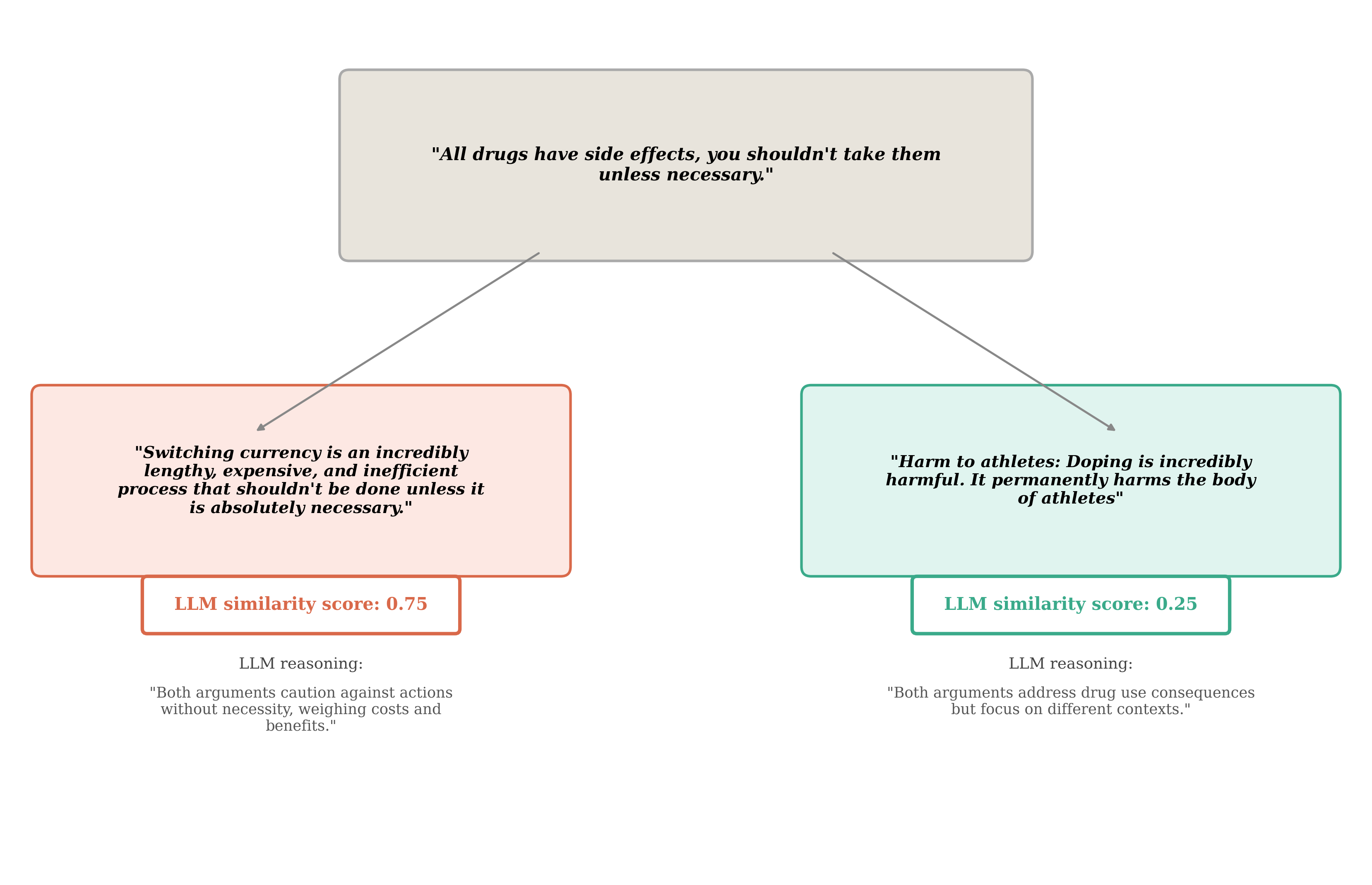}
\caption*{\textit{Notes}: \exhibit}
\label{fig:exhibit_figure}
\end{figure}

Notably, although the distinction between the pairs is clear to a human, the contextualized word embeddings represent the second pair as more similar vectors than the first pair as witnessed by a higher similarity score, causing the classifier to make errors.\\

This example is an instance of the point made by several authors (e.g. \cite{Liu2023}) that contextualized word embeddings might not be the optimal way to represent arguments. To humans, ``Argumentative reasoning is based on justifying a plausible conclusion with arguments in its favour'' \citep[p. iv]{David2021}. At its core, therefore, argument similarity rests ultimately in assessing the proximity of logical reasoning. In addition, similarity metrics based on the cosine similarity of embeddings have additional problems due to the lack of a notion of the size of an embedding \citep{Steck2024}.\\

In sum, while contextualized word embeddings are a great general-purpose technology for NLP tasks, they may be suboptimal for argument learning tasks. Moreover, their derived similarity metrics often fail to comply with desirable properties for such tools, such as the basic monotonicity requirement (i.e., two more similar arguments should receive a higher similarity score than two less similar arguments). This, in turn, can lead to errors in downstream tasks.\\

In this paper, we propose a new framework to represent and assess similarity between arguments:  \textit{logical embeddings}. As its name implies, \textit{logical embeddings} provide a way to encode only the logical content of arguments, thus avoiding many shortcomings of contextualized word embeddings.  This approach has three desirable properties. First,  \textit{logical embeddings} are \textit{optimal embeddings} in the sense that no logical information is lost in the representation. Second, the supporting similarity metric is ensured to satisfy theoretically desirable properties of a proximity notion, such as monotonicity and the triangle inequality. Third, as we will illustrate, logical embeddings provide a powerful tool to analyze complex argumentation schemes.\\


Our starting point is to introduce into the argument similarity literature a notion of similarity from mathematical logic \citep{David2021,simpson2013mathematical}. We outline the full necessary mathematical structure in the next section. In essence, this metric builds upon the definition of an argument as a logical object, composed of a set of premises and a conclusion, with a logical association that allows the latter to follow from the former (for a similar treatment from linguistics, see \cite{govier2013practical}). Argument similarity is then defined as the fraction of overlapping premises, modulo logical equivalence, together with the overlap of the logical consequences derived from its conclusion. The resulting similarity metric has numerous theoretical properties, among which are monotonicity and triangle inequality, that allow it to \textit{metricize} logic. Following \citet{David2021} and \citet{jaccard1908nouvelles}, we call this similarity measure the Syntactic-Semantic Jaccard measure.\\


We prove that the Syntactic-Semantic Jaccard similarity measure is a positive semi-definite kernel on the set of arguments. Following the literature on Reproducing Kernel Hilbert Spaces  (RKHS) \citep{Aronszajn1950}, such a property defines unique embeddings of arguments via their premise-conclusion decomposition. This gives rise to \textit{logical embeddings}. We further prove that, because of the desirable properties of the similarity metric that spans them, \textit{logical embeddings} preserve all logical information. This stems from the fact that the Syntactic-Semantic Jaccard similarity measure is a characteristic kernel \citep{fukumizu2004dimensionality, Fukumizu2008}, which guarantees an injective map from an argument to its embedding. As with other RKHS applications, \textit{logical embeddings} can be used for many supervised and unsupervised tasks (see \citet{Shawe-Taylor2004} and references therein).\\


We discuss how to implement the proposed framework in practice, and test performance. In short, we propose a three-step procedure. First, frontier Large Language Models (LLMs) can productively be used to decompose an argument into its constituent premises and conclusion. Second, an LLM maps a conclusion into its logical consequences. Third, we then train an entailment model to measure pairwise entailment among premises and (sub)conclusions. Steps 1 and 2 are essentially retrieval tasks, and standard methods perform adequately for this purpose. Our next steps involve decomposing conclusions into their logical consequences. This is important and best illustrated by example. The conclusions "The earth is round" and "The earth is round and blue" are not logically equivalent but have common implications—both imply that the earth is round—making them partially similar. LLMs perform very well on this task, and we propose an implementation. Finally, we train a standard entailment model to compute logical similarity between premises and conclusions in their Conjunctive Normal Form (CNF). These steps allow for the computation of the Syntactic-Semantic Jaccard measure, thereby enabling the practical use of \textit{logical embeddings}.\\

Finally, we test the performance of logical embeddings on a standard classification task, using the IBM-ArgQ-6.3kArgs corpus \citep{toledo2019}. Using this corpus, we construct a Gram matrix of pairwise argument-similarity scores and use its reduced-dimensional representation as a logical embedding for each argument. We compare these embeddings against standard semantic embedding baselines, including averaged GloVe vectors, BERT-base, RoBERTa-base, SBERT, and OpenAI text-embedding-3-small embeddings. Across logistic and linear regressions, as well as a small neural network and a random forest, embeddings outperform the semantic baselines. We then concatenate logical embeddings with each semantic embedding to combine logical and semantic information. The best performance is obtained either by logical embeddings alone or by logical embeddings combined with OpenAI text-embedding-3-small embeddings.



Our paper makes three key contributions. First, we introduce logical embeddings. We show that the minimal structure imposed on similarity computation by imposing the premise-conclusion structure and the Syntactic-Semantic Jaccard measure provide several desirable properties for resulting similarity score. Second, we prove that the similarity matrix resulting from the Syntactic-Semantic Jaccard measure on a set of arguments is an injective positive semi-definite kernel, and we, therefore, prove the optimality of the `logical embeddings' induced by this map. Finally, we propose an implementation pipeline and test the performance of logical embeddings against standard benchmarks.

\section{Logical embeddings}
\label{S:logical_embeddings}

This section develops the framework used in the rest of the paper. The construction proceeds in three steps. First, we model an argument as a collection of premises paired with a conclusion. Second, we define a logic-aware similarity function on pairs of arguments. Third, we prove that this similarity function is a positive semidefinite kernel. Accordingly, this kernel induces a feature map into a reproducing kernel Hilbert space (RKHS). We call that feature map a logical embedding.

\subsection{Arguments as premise--conclusion pairs}

We work with a propositional logic \((\mathcal L,\vdash)\), where \(\mathcal L\) is the language of formulas and \(\vdash\) is the logical consequence relation. Appendix~\ref{A:logic_preliminaries} gives the formal construction of these objects.

The definitions in this section use only the standard syntax of propositional logic. Atomic propositions are statements that can be assigned a truth value. Formulas are built from atomic propositions using the Boolean connectives: negation \((\neg)\), conjunction \((\land)\), disjunction \((\lor)\), implication \((\rightarrow)\), and biconditionality \((\leftrightarrow)\). Thus, if \(p\) and \(q\) are atomic propositions, then \(p\land q\), \(p\rightarrow q\), and \(\neg p\lor q\) are formulas. Finally, a literal is an atomic proposition or its negation.

The relation \(\vdash\) denotes logical consequence. For
\(\Gamma\subseteq\mathcal L\) and \(\phi\in\mathcal L\), the expression
\(\Gamma\vdash\phi\) means that every truth assignment satisfying all formulas in
\(\Gamma\) also satisfies \(\phi\).

For formulas \(\phi,\psi\in\mathcal L\), we say that \(\phi\) and \(\psi\) are logically equivalent, and write \(\phi\equiv\psi\), when they entail each other: \(\phi\vdash\psi\) and \(\psi\vdash\phi\). Here \(\phi\vdash\psi\) abbreviates \(\{\phi\}\vdash\psi\). Equivalently, \(\phi\equiv\psi\) if, and only if, \(\phi\) and \(\psi\) have the same truth conditions. We denote the equivalence class of \(\phi\) under \(\equiv\) by \([\phi]_{\equiv}=\{\psi\in\mathcal L:\psi\equiv\phi\}\).

\begin{definition}[Argument]
An argument is a pair
\[
a=(\Phi,\phi),
\]
where \(\Phi\subseteq\mathcal L\) is a finite set of premises and
\(\phi\in\mathcal L\) is a conclusion such that \(\Phi\vdash\phi\). The premise
set is required to be consistent, meaning that premises do
not contradict each other, and non-redundant, meaning that no premise is
dispensable for deriving \(\phi\). The set of all arguments is denoted
\(\operatorname{Arg}(\mathcal L)\).
\end{definition}

The point of this representation is to map natural-language arguments into a common premise--conclusion structure that captures their logical content. Under this structure, an argument is represented by the premises it uses and the conclusion those premises support. This makes it possible to compare arguments on the basis of their logical components rather than their wording. Appendix~\ref{A:logic_preliminaries} details the formal consistency and
non-redundancy conditions. Section~\ref{S:nlp_implementation} describes how we approximate this structure from text using NLP models.

This perspective determines an equivalence relation on arguments. Two arguments are identical for our purposes when they use the same premises up to logical equivalence and have logically equivalent conclusions. For a set of formulas \(\Phi\), define
\[
[\Phi]_{\equiv}=\{[\xi]_{\equiv}:\xi\in\Phi\},
\]
the set of equivalence classes represented by the formulas in \(\Phi\).

\begin{definition}[Argument equivalence]
For arguments \(a=(\Phi,\phi)\) and \(b=(\Psi,\psi)\), we write \(a\approx b\) if
\[
[\Phi]_{\equiv}=[\Psi]_{\equiv}
\qquad\text{and}\qquad
\phi\equiv\psi .
\]
\end{definition}

Thus, two arguments may be expressed differently and still be equivalent. For example, the premises ``John is Susan's brother'' and ``Susan is John's sister'' are worded differently but have the same truth conditions: each is true if and only if the other is true. The equivalence relation therefore identifies arguments by mutual entailment, not by wording.

\subsection{A syntactic--semantic similarity between arguments}
\label{S: ConstructingSim}

We now define the similarity function that measures logical proximity between arguments and induces the logical embedding. The construction follows \citet{David2021} and \citet{simpson2013mathematical}. It compares arguments along two dimensions: the premises they use and the logical content of their conclusions.

The main difficulty lies in comparing the conclusions. We do not compare formulas
by the full set of their logical consequences, since that set is typically
infinite and contains many equivalent reformulations. Instead, following
\citet{David2021}, we compare conclusions by fixing a finite set of non-redundant consequences.

We use conjunctive normal form (CNF) to fix these representatives. A formula is
in CNF if it is a conjunction of one or more clauses, each of which is a
disjunction of one or more literals. Informally, a CNF formula is obtained by
joining ``or'' statements with ``and'' connectives. For example,
\(p\to(q\land r)\) is logically equivalent to the CNF formula
\[
(\neg p\lor q)\land(\neg p\lor r).
\]
This restriction loses no expressive power: every propositional formula is
logically equivalent to some CNF formula \citep{russell1995modern}.

For a formula \(\phi\), let \(\operatorname{CN}_{\mathcal F}(\phi)\) be its set of logical consequences in CNF form. This is the finite set of CNF formulas that are entailed by \(\phi\), use only the literals relevant to \(\phi\), and are listed only once up to logical equivalence.\footnote{Appendix~\ref{A:logic_preliminaries} gives the formal construction of this set.}

For example,
\[
\operatorname{CN}_{\mathcal F}(p\land q)
=
\{p,q,p\lor q,p\land q\}.
\]
Indeed, \(p\land q\) entails \(p\), entails \(q\), entails \(p\lor q\), and entails itself. In this case, the relevant literals are \(p\) and \(q\), and the displayed set lists the distinct CNF consequences over those literals.

Given two arguments \(a=(\Phi,\phi)\) and \(b=(\Psi,\psi)\), define their premise similarity, or syntactic similarity, by the Jaccard overlap of their premise equivalence classes:
\[
s_{\mathrm{syn}}(\Phi,\Psi)
=
\frac{
|[\Phi]_{\equiv}\cap[\Psi]_{\equiv}|
}{
|[\Phi]_{\equiv}\cup[\Psi]_{\equiv}|
},
\]
with the convention that the ratio is \(1\) when both sets are empty.

Define the conclusion similarity, or semantic similarity, by the Jaccard overlap of the sets of logical consequences in CNF form:

\[
s_{\mathrm{sem}}(\phi,\psi)
=
\frac{
|\operatorname{CN}_{\mathcal F}(\phi)
\cap
\operatorname{CN}_{\mathcal F}(\psi)|
}{
|\operatorname{CN}_{\mathcal F}(\phi)
\cup
\operatorname{CN}_{\mathcal F}(\psi)|
},
\]
again with the convention that the ratio is \(1\) when both sets are empty.

The syntactic--semantic Jaccard similarity of $a$ and $b$ is the convex
combination
\[
\operatorname{sim}^{\sigma}(a,b)
=
\sigma s_{\mathrm{syn}}(\Phi,\Psi)
+
(1-\sigma)s_{\mathrm{sem}}(\phi,\psi),
\qquad
0<\sigma<1.
\]
 Accordingly, the similarity measure has a direct interpretation: two arguments are similar when they rely on logically equivalent premises and when their conclusions have logical consequences in common. The parameter $\sigma$ controls the relative weight placed on shared premises
versus shared conclusions.

\citet{David2021} shows that the syntactic--semantic Jaccard similarity satisfies two groups of desirable properties for comparing logical arguments. The first group concerns general consistency: the measure is maximized on identical arguments, symmetric in its arguments, compatible with a triangle-inequality principle, and invariant under substitution of maximally similar arguments. The second group concerns sensitivity to argumentative content: the measure is zero when arguments share no relevant content, positive when they share some content, and strictly increases with greater overlap in premises or in the logical consequences of their conclusions. Appendix~\ref{A: SimNotions} states the full list of properties formally.

The following result is a direct consequence of \citet[Corollary~1]{Amgoud2018}.\\

\begin{theorem}[Logical faithfulness]
\label{T: Metrics}
For any \(0<\sigma<1\) and any arguments
\(a,b\in\operatorname{Arg}(\mathcal L)\),
\[
\operatorname{sim}^{\sigma}(a,b)=1
\quad\text{if and only if}\quad
a\approx b .
\]
\end{theorem}

Theorem~\ref{T: Metrics} shows that the similarity measure is faithful to the
underlying logical representation of arguments. Two arguments have maximal similarity if and only if they are logically equivalent: they have the same premises,
up to logical equivalence, and logically equivalent conclusions. Thus,
\(\operatorname{sim}^{\sigma}\) treats argument equivalence as the condition for
maximal similarity, while still assigning graded values to partial overlap in
premises and conclusion consequences.

This property distinguishes the proposed similarity from generic embedding similarities, such as cosine similarity between text embeddings. Indeed, similarity is determined by logical equivalence, not by word choice, topic overlap, or stylistic resemblance. The measure therefore attends to the logical structure of the argument rather than to its  wording.

\subsection{From similarity to logical embeddings}

We now show that the syntactic--semantic Jaccard similarity induces an embedding
of arguments into a Hilbert space. The key step is the following additional property of the similarity measure.\\

\begin{theorem}[Kernel property]
\label{T:Kernel}
For any \(0<\sigma<1\), \(\operatorname{sim}^{\sigma}\) is a positive
semidefinite kernel on \(\operatorname{Arg}(\mathcal L)\).
\end{theorem}

By the Moore--Aronszajn theorem \citep{Aronszajn1950}, the similarity measure \(\operatorname{sim}^{\sigma}\) determines a
reproducing kernel Hilbert space \(\mathcal H_{\sigma}\) and a canonical feature
map
\[
\Theta_{\sigma}:\operatorname{Arg}(\mathcal L)\to\mathcal H_{\sigma}
\]
such that, for all \(a,b\in\operatorname{Arg}(\mathcal L)\),
\[
\left\langle
\Theta_{\sigma}(a),\Theta_{\sigma}(b)
\right\rangle_{\mathcal H_{\sigma}}
=
\operatorname{sim}^{\sigma}(a,b).
\]
We call \(\Theta_{\sigma}(a)\) the logical embedding of the argument \(a\).
Appendix~\ref{A:rkhs} gives a brief review of reproducing kernel Hilbert spaces.

This construction gives the embedding a direct interpretation. The vector \(\Theta_{\sigma}(a)\) encodes the logical content of \(a\): two arguments receive the same embedding exactly when they are equivalent under the premise--conclusion representation. This encoding is relational. The position of \(a\) in the Hilbert space is determined by how \(a\) relates to every other possible argument through the similarity function \(\operatorname{sim}^{\sigma}(a,\cdot)\). Thus, the geometry of the embedding is governed by the premise-conclusion structure of arguments.

Moreover, the RKHS embedding admits a finite-sample approximation. Given a corpus of arguments \(a_1,\ldots,a_n\), the formal construction defines its Gram matrix, that is, the matrix of pairwise similarities
\[
K^\sigma_{ij}
=
\operatorname{sim}^{\sigma}(a_i,a_j)
=
\left\langle
\Theta_{\sigma}(a_i),\Theta_{\sigma}(a_j)
\right\rangle_{\mathcal H_{\sigma}}.
\]
In practice, we estimate this matrix from natural-language text by approximating the premise--conclusion structure of each argument and the corresponding syntactic--semantic similarity. Let \(\widehat K^\sigma\) denote the resulting empirical kernel matrix. We then apply kernel PCA to \(\widehat K^\sigma\) to obtain finite-dimensional coordinates for the observed arguments.\footnote{Appendix~\ref{A:kernel_pca} reviews the theoretical basis of this procedure.} These coordinates provide a low-dimensional approximation to the ideal RKHS embeddings on the observed corpus.

Finally, the logical embedding preserves all information relevant to the formal argument representation, up to logical equivalence. In other words, arguments that are not equivalent under \(\approx\) receive distinct embeddings, whereas equivalent arguments receive the same embedding. The next theorem formalizes this property.\\

\begin{theorem}[No loss of logical information]
\label{T: CKernel}
For any $0<\sigma<1$, let $\mathcal H_{\sigma}$ be the RKHS induced by
$\operatorname{sim}^{\sigma}$, and let
\[
\Theta_{\sigma}:
\operatorname{Arg}(\mathcal L)
\to
\mathcal H_{\sigma},
\qquad
\Theta_{\sigma}(a)=\operatorname{sim}^{\sigma}(\cdot,a)
\]
be the canonical feature map. Then, for any
$a,b\in\operatorname{Arg}(\mathcal L)$,
\[
\Theta_{\sigma}(a)=\Theta_{\sigma}(b)
\quad\text{if and only if}\quad
a\approx b.
\]
Consequently, $\Theta_{\sigma}$ induces an injective map
\[
\overline{\Theta}_{\sigma}:
\operatorname{Arg}(\mathcal L)/{\approx}
\to
\mathcal H_{\sigma},
\qquad
\overline{\Theta}_{\sigma}([a]_{\approx})
=
\Theta_{\sigma}(a).
\]
\end{theorem}

\section{NLP implementation of logical embeddings}\label{S:nlp_implementation}

While logical embeddings and their supporting similarity measures have desirable theoretical properties, their utility is limited without implementation to real data. This section discusses how to compute the logic-based similarity measures from Section \ref{S: ConstructingSim} using existing NLP and Machine Learning tools.\\

In essence, we propose a three-step procedure to computing Syntactic-Semantic Jaccard similarity measures:

\begin{enumerate}
    \item Training an entailment model.
    \item Using a large language model (LLM) to decompose an argument into its constituent premises and conclusion.
    \item Using an LLM to map a conclusion into its CNF.
\end{enumerate}

We further detail each step in detail in the upcoming sections.

\subsection{Training an entailment model}

The key to establishing logical overlap is determining whether two propositions, $p$ and $q$, are logically equivalent. This task can be further decomposed into establishing that $p$ entails $q$ and that $q$ entails $p$. Thus, a model that can determine whether the first proposition in an ordered pair logically implies the second can also predict logical equivalence.\\

Establishing logical equivalence is a well-defined NLP task known as Recognizing Textual Entailment (RTE), first introduced by \citet{maccartney-manning-2008-modeling}. This task has received significant attention and benefits from large training datasets such as the Stanford Natural Language Inference (SNLI) dataset, which contains 570,000 human-written, manually labeled English sentence pairs. Moreover, existing Machine Learning tools have proven highly effective at this task, achieving accuracies exceeding $0.9$ in several studies \citep{wang2021entailmentfewshotlearner,pilault2022conditionallyadaptivemultitasklearning,zhang2020semanticsawarebertlanguageunderstanding}. Consequently, logical entailment classification can be satisfactorily accomplished via zero-shot learning or fine-tuning of these models.

\subsection{Argument decomposition}

The rest of the paper treats this construction as an ideal target. In practice,
natural-language arguments do not come with explicit premise sets, formal
conclusions, or canonical consequence sets. Our empirical task is therefore to
approximate the ideal logical embedding from text and to evaluate whether the
resulting representations are useful for downstream argument-related tasks.

Premise and conclusion extraction is a well-studied task in argumentation NLP settings \citep{Palau2011,peldszus-stede-2015-joint,Stab2017}. However, existing approaches often rely on sophisticated models that can now be superseded by large language models (LLMs). For example, we have found good results using a chained prompt as follows:\\

\boxed{
\begin{minipage}{\dimexpr\columnwidth-2\fboxsep-2\fboxrule\relax}
\underline{Phase 1:}\\

You are a skilled logician. Given a text, extract its conclusion. A conclusion is a logical result of the relationship between the premises. Conclusions serve as the thesis of the argument. \\

\textit{Inputs:} Text \\

\underline{Phase 2:}\\

You are a skilled logician. Given a text and its conclusion, extract its premises. \\

\textit{Inputs:} Text + Conclusion from Phase 1.
\end{minipage}
}

\subsection{Conclusion CNF formulation}

The last step consists of mapping a conclusion into its CNF representation. This is the hardest subtask, but can also be achieved with high precision by modern LLMs. In our applications we have used the following prompt with success.\\

\boxed{
\begin{minipage}{\dimexpr\columnwidth-2\fboxsep-2\fboxrule\relax}

Convert the following text into its Conjunctive Normal Form (CNF). Identify the individual propositions and combine them into disjunctions of complete and meaningful propositions. Ensure each proposition contains at least a noun and a verb and is unique. \\

\textit{Inputs:} Conclusion \\

\end{minipage}
}

\subsection{Putting everything together}

Given two arguments, $a_1$ and $a_2$, we can now compute their logic-based similarity using the previously explored tools. The process unfolds as follows:

\begin{itemize}
    \item Decompose $a_1$ and $a_2$ into their constituent premises and conclusions, $(\Phi_1, \phi_1)$ and $(\Phi_2, \phi_2)$.
    \item Find the CNF representation of $\phi_1$ and $\phi_2$, and calculate $\text{CN}_{\mathcal{F}}(\phi_1)$ and $\text{CN}_{\mathcal{F}}(\phi_2)$.
    \item Using the entailment model, calculate $s_{\text{syn}}(\Phi_1, \Phi_2)$ and $s_{\text{sem}} (\phi_1, \phi_2)$ by assessing the overlap, modulo logical equivalence, between $\Phi_1$ and $\Phi_2$, and $\text{CN}_{\mathcal{F}}(\phi_1)$ and $\text{CN}_{\mathcal{F}}(\phi_2)$.
    \item Fix a $\sigma \in [0,1]$ and aggregate into $\text{sim}^{\sigma}(a_1, a_2)$.
\end{itemize}

Logical embeddings can then be produced and exploited via the reproducibility property of the similarity measure. In the Appendix, Figure \ref{fig:exhibit_gram_77_vs_4612} we provide an example of how similarity is constructed for a pair of arguments.

\subsection{Performance evaluation}

We evaluate the performance of our embeddings on the standard IBM-ArgQ-6.3kArgs dataset \citep{toledo2019}. This dataset contains arguments in favor or against several topics. To economize on implementation costs we focus on the topics of doping, vaccins, and cryptocurrencies . For each embedding-classifier pair, we use a fixed 70/30 train-test split and report binary F1 on the test set. The neural-network classifier uses five-fold cross-validation within the training set for hyperparameter selection, while the other classifiers are fit directly on the training split. We report test set F1-scores for different embedding methods as rows and different models as columns. We compare our logical embeddings in the first row against several pretrained alternatives such as GloVe, BERT, SBERT and OpenAI text-embedding-3-small embeddings. We also evaluate concatenated representations that combine each pretrained embedding with the logical embedding. Classifiers include L1-penalized logistic regression, Lasso and Ridge linear models with thresholded predictions, a one-hidden-layer feed-forward neural network, and a random forest. The neural network has a ReLU hidden layer and sigmoid output; its learning rate, weight decay, and hidden dimension are selected by five-fold cross-validation over the training set. The random forest uses 200 trees. Table ~\ref{TABLE:F1 results} shows that the Logical Embeddings outperform every model when compared side by side. In combination, Open AI text-embedding-3-small and logical embeddings outperform all other combinations. 

\section{Conclusion}

In this paper we introduced and tested \textit{logical embeddings}. We show that these are optimal embeddings and provide implementation details. Implementation code is available at \nolinkurl{https://github.com/lheldring/logical_embeddings}.

\begin{table}[h!]
\footnotesize
\centering
\captionsetup{justification=centering}
\caption{\textsc{F1 scores 100-dimensional logical embeddings}}\label{TABLE:F1 results}
\resizebox{\columnwidth}{!}{%
\begin{tabular}{lccccc}
\toprule\toprule
\emph{Estimation:} & \multicolumn{5}{c}{F1 Scores}\\\cmidrule(lr){2-6}
& \multicolumn{1}{c}{Logistic} & \multicolumn{1}{c}{Lasso} & \multicolumn{1}{c}{Ridge} &  \multicolumn{1}{c}{Neural Network} & \multicolumn{1}{c}{Random Forest} \\
\midrule\\
\input{Figures/table_f1_comparison_pairwise}
\end{tabular}
}
\captionsetup{font={scriptsize}, width={\columnwidth}, justification=justified}
\caption*{\textit{Notes}: \mainresults }
\end{table}

\newpage
\clearpage

\bibliographystyle{plainnat}
\bibliography{FinalVersion.bib}


\appendix

\section{Appendix / supplemental material}\label{paper_appendix}

\subsection{Logical preliminaries}
\label{A:logic_preliminaries}

This appendix gives the formal logical background used in
Section~\ref{S:logical_embeddings}. The main text introduces only the notation
needed to define arguments, logical equivalence, and the syntactic--semantic
similarity. Here we state the underlying propositional language, consequence
relation, CNF convention, and canonical consequence sets more explicitly. The
presentation follows \citet{amgoud2021compilation}, \citet{David2021}, and
\citet{simpson2013mathematical}.

\subsubsection{Sentential logic}

Sentential logic starts from atomic propositions and builds more complex
statements from them using Boolean connectives.

\begin{definition}[Atoms and sentential language]
An atomic proposition, or atom, is a statement that can be assigned a truth
value. A sentential language is a finite set
\[
L=\{p,q,r,s,\ldots\}
\]
of atoms.
\end{definition}

Atoms are the basic units of the language. For example, an atom \(p\) may stand
for a truth-valued statement such as ``the policy reduces debt.'' The language
\(L\) specifies which atoms are available in the application.

\begin{definition}[Propositional connectives]
The propositional connectives are
\[
\neg,\quad \land,\quad \lor,\quad \rightarrow,\quad \leftrightarrow,
\]
denoting negation, conjunction, disjunction, implication, and biconditionality,
respectively.
\end{definition}

Connectives combine atoms into more complex truth-valued expressions. These
expressions are called formulas.

\begin{definition}[Formulas]
The set of well-formed formulas generated by \(L\), denoted \(\mathcal L\), is
the smallest set satisfying the following conditions:
\begin{enumerate}
    \item if \(p\in L\), then \(p\in\mathcal L\);
    \item the truth constants \(\top\) and \(\bot\) belong to \(\mathcal L\);
    \item if \(\phi\in\mathcal L\), then \(\neg\phi\in \mathcal L\);
    \item if \(\phi,\psi\in \mathcal L\), then
    \[
    (\phi\land\psi),\quad
    (\phi\lor\psi),\quad
    (\phi\rightarrow\psi),\quad
    (\phi\leftrightarrow\psi)
    \]
    are in \(\mathcal L\).
\end{enumerate}
\end{definition}

Thus, atoms are formulas, and formulas can be combined recursively to form more
complex formulas. For instance, if \(p,q\in L\), then \(p\land q\),
\(p\rightarrow q\), and \(\neg p\lor q\) are formulas. A literal is an atom or
the negation of an atom.

Next, we define truth assignments. A truth assignment determines which formulas
are true and which are false, subject to the usual truth tables for the Boolean
connectives. Fix the set of truth values \(\{F,T\}\), where \(F\) denotes
falsity and \(T\) denotes truth.

\begin{definition}[Valuation]
A valuation is a map
\[
v:\mathcal L\to\{F,T\}
\]
that satisfies the usual truth-functional rules. In particular,
\[
v(\top)=T
\qquad\text{and}\qquad
v(\bot)=F,
\]
\[
v(\neg\phi)=T \quad\text{if and only if}\quad v(\phi)=F,
\]
\[
v(\phi\land\psi)=T \quad\text{if and only if}\quad
v(\phi)=T \text{ and } v(\psi)=T,
\]
\[
v(\phi\lor\psi)=T \quad\text{if and only if}\quad
v(\phi)=T \text{ or } v(\psi)=T,
\]
\[
v(\phi\rightarrow\psi)=T \quad\text{if and only if}\quad
v(\phi)=F \text{ or } v(\psi)=T,
\]
and
\[
v(\phi\leftrightarrow\psi)=T \quad\text{if and only if}\quad
v(\phi)=v(\psi).
\]
\end{definition}

The preceding definitions specify the syntax of the language and the truth
conditions for its formulas. We now package these objects into the propositional
logic used throughout the paper.

\begin{definition}[Sentential logic]
A sentential, or propositional, logic is a pair \((\mathcal L,\vdash)\), where
\(\mathcal L\) is the set of well-formed formulas generated by a sentential
language \(L\), and \(\vdash\) is the semantic consequence relation. For
\(\Gamma\subseteq\mathcal L\) and \(\varphi\in\mathcal L\), we write
\[
\Gamma\vdash\varphi
\]
if and only if every valuation that assigns \(T\) to all formulas in \(\Gamma\)
also assigns \(T\) to \(\varphi\).
\end{definition}

Fixing a logic \((\mathcal L,\vdash)\) lets us define when two formulas have the
same logical content.

\begin{definition}[Logical equivalence of formulas]
Two formulas \(\phi,\psi\in\mathcal L\) are logically equivalent if
\[
\phi\vdash\psi
\qquad\text{and}\qquad
\psi\vdash\phi,
\]
where \(\phi\vdash\psi\) abbreviates \(\{\phi\}\vdash\psi\). In this case, we
write \(\phi\equiv\psi\).
\end{definition}

Logical equivalence means that two formulas entail each other. Equivalently, they have the same truth conditions.

Moroever, logical equivalence is an equivalence relation: it is reflexive, symmetric, and
transitive. We can therefore group formulas into equivalence classes.

\begin{definition}[Equivalence class of a formula]
For a formula \(\phi\in\mathcal L\), its equivalence class under logical
equivalence is
\[
[\phi]_{\equiv}
=
\{\psi\in\mathcal L:\psi\equiv\phi\}.
\]
\end{definition}

Equivalence classes let us identify formulas that differ syntactically but have
the same logical content. To choose representatives of these classes, we use
conjunctive normal form.

\begin{definition}[Conjunctive normal form]
A literal is an atom \(p\in L\) or its negation \(\neg p\). A clause is a
finite disjunction of literals. A formula \(\phi\in\mathcal L\) is in
conjunctive normal form (CNF) if it can be written as
\[
\phi=\bigwedge_{i=1}^{m} c_i,
\]
where each clause \(c_i\) has the form
\[
c_i=\bigvee_{j=1}^{n_i}\ell_{ij},
\]
and each \(\ell_{ij}\) is a literal. Equivalently, a CNF formula is a
conjunction of clauses, each of which is a disjunction of literals.
\end{definition}

Every propositional formula is logically equivalent to some CNF formula
\citep{russell1995modern}. For example,
\[
p\to(q\land r)
\equiv
(\neg p\lor q)\land(\neg p\lor r).
\]

Following \citet{amgoud2021compilation}, we use CNF representatives to
avoid redundant reformulations. The goal is to keep one representative formula
for each relevant equivalence class while preserving the literals of the
original formula.

\begin{definition}[Finite CNF language]
Let \(\mathcal F\subseteq\mathcal L\). We say that \(\mathcal F\) is a finite
CNF language for \(\mathcal L\) if, for every \(\phi\in\mathcal L\), there
exists a unique \(\psi\in\mathcal F\) such that
\begin{enumerate}
    \item \(\psi\in[\phi]_{\equiv}\);
    \item \(\operatorname{Lit}(\psi)=\operatorname{Lit}(\phi)\);
    \item \(\psi\) is in CNF.
\end{enumerate}
\end{definition}

Thus, \(\mathcal F\) contains a unique CNF representative for each formula, up to
logical equivalence. We write
\(\operatorname{CNF}_{\mathcal F}(\phi)\) for the unique element of
\(\mathcal F\) that is logically equivalent to \(\phi\) and has the same literals
as \(\phi\).

 \subsection{Logical arguments}

We now define the logical arguments used in the main text, following
\citet{Besnard2001} and \citet{Amgoud2018}. A logical argument consists of a
set of premises and a conclusion that follows from those premises. Before giving
the definition, we record two auxiliary notions for sets of formulas.

\begin{definition}[Consistency]
A set of formulas \(\Phi\subseteq\mathcal L\) is consistent if
\[
\Phi\nvdash\bot.
\]
It is inconsistent otherwise.
\end{definition}

Equivalently, \(\Phi\) is consistent if there exists a valuation under which all
formulas in \(\Phi\) are true.

We also need to compare premise sets up to logical equivalence.

\begin{definition}[Logical equivalence of sets of formulas]
For \(\Phi,\Psi\subseteq\mathcal L\), write
\[
[\Phi]_{\equiv}=\{[\phi]_{\equiv}:\phi\in\Phi\}.
\]
We say that \(\Phi\) and \(\Psi\) are logically equivalent, written
\(\Phi\cong\Psi\), if
\[
[\Phi]_{\equiv}=[\Psi]_{\equiv}.
\]
\end{definition}

We now define the logical arguments used in the main text, following
\citet{Besnard2001} and \citet{Amgoud2018}. A logical argument consists of a
finite set of premises and a conclusion that follows from those premises. The
premise set is required to be consistent and non-redundant.

\begin{definition}[Logical argument]
\label{Def:argument}
An argument in the logic \((\mathcal L,\vdash)\) is a pair
\[
a=(\Phi,\phi),
\]
where \(\Phi\subseteq\mathcal L\) is finite, \(\phi\in\mathcal L\), and:
\begin{enumerate}
    \item \(\Phi\) is consistent;
    \item \(\Phi\vdash\phi\);
    \item \(\Phi\) is non-redundant, meaning that there is no
    \(\Phi'\subsetneq\Phi\) such that \(\Phi'\vdash\phi\).
\end{enumerate}
\end{definition}

The set \(\Phi\) is the set of premises of \(a\), and \(\phi\) is its
conclusion. We write
\[
\operatorname{Prem}(a)=\Phi
\qquad\text{and}\qquad
\operatorname{Conc}(a)=\phi.
\]
The set of all arguments over \(\mathcal L\) is denoted
\(\operatorname{Arg}(\mathcal L)\).\\

This representation lets us define when two arguments have the same logical
content. The definition compares premise sets up to logical equivalence and
compares conclusions by logical equivalence.

\begin{definition}[Argument equivalence]
\label{D:ArgEquivalence}
Two arguments \(a,b\in\operatorname{Arg}(\mathcal L)\) are equivalent, written
\(a\approx b\), if
\[
\operatorname{Prem}(a)\cong \operatorname{Prem}(b)
\qquad\text{and}\qquad
\operatorname{Conc}(a)\equiv \operatorname{Conc}(b).
\]
\end{definition}

 \subsection{Measuring similarity}\label{S: MeasuringSim}

The logical representation of arguments makes it possible to compare their
constituent parts separately. This is the basis of the similarity measures
developed by \citet{Amgoud2018} and \citet{David2021}. These measures compare
arguments by comparing their premises and their conclusions, and they satisfy
the rationality properties stated in Appendix~\ref{A: SimNotions}.

We first need an operator for comparing sets of formulas modulo logical
equivalence.

\begin{definition}[Common formulas operator]
For sets of formulas \(\Phi,\Psi\subseteq\mathcal L\), define
\[
\operatorname{Co}(\Phi,\Psi)
=
\{\phi\in\Phi:\text{ there exists }\psi\in\Psi\text{ such that }\phi\equiv\psi\}.
\]
\end{definition}

The operator \(\operatorname{Co}(\Phi,\Psi)\) selects the formulas in \(\Phi\)
that also occur in \(\Psi\), up to logical equivalence. Thus,
\(\Phi\cong\Psi\) if and only if
\[
\operatorname{Co}(\Phi,\Psi)=\Phi
\qquad\text{and}\qquad
\operatorname{Co}(\Psi,\Phi)=\Psi.
\]

We next define the consequence
sets used to compare conclusions.

\begin{definition}[Logical consequences of a formula]
For a formula \(\phi\in\mathcal L\), define
\[
\operatorname{CN}(\phi)
=
\{\psi\in\mathcal L:\phi\vdash\psi\}.
\]
This is the set of all logical consequences of \(\phi\).
\end{definition}

The set \(\operatorname{CN}(\phi)\) is typically infinite. For example,
\(p\vdash p\), \(p\vdash p\land p\), \(p\vdash (p\land p)\land p\), and so on.
Moreover, many such consequences are redundant up to logical equivalence. To
obtain a finite set for comparison, we restrict attention to CNF representatives
that use only literals on which \(\phi\) depends.

Let \(L^{\pm}\) denote the set
of literals, that is, atoms and their negations. For a formula \(\phi\), let
\(\operatorname{Lit}(\phi)\) denote the set of literals occurring in \(\phi\)
after rewriting \(\phi\) in negation normal form.

A formula \(\phi\) is independent of a literal \(\ell\in L^{\pm}\) if there
exists a formula \(\chi\in\mathcal L\) such that \(\chi\equiv\phi\) and
\(\ell\notin\operatorname{Lit}(\chi)\). Otherwise, \(\phi\) depends on
\(\ell\). We write \(\operatorname{DepLit}(\phi)\) for the set of literals on
which \(\phi\) depends.

Fix a finite CNF representative language \(\mathcal F\subseteq\mathcal L\).
That is, \(\mathcal F\) contains one CNF representative from each logical
equivalence class, and each representative is written with no literals beyond
those on which it depends.

\begin{definition}[Finite CNF consequence set]
For a formula \(\phi\in\mathcal L\), define
\[
\operatorname{CN}_{\mathcal F}(\phi)
=
\{\psi\in\mathcal F:
\phi\vdash\psi
\text{ and }
\operatorname{Lit}(\psi)\subseteq \operatorname{DepLit}(\phi)\}.
\]
\end{definition}

Thus, \(\operatorname{CN}_{\mathcal F}(\phi)\) contains the CNF representatives
of the logical consequences of \(\phi\) that use only literals on which
\(\phi\) depends. The restriction to \(\mathcal F\) removes equivalent
reformulations; the restriction to \(\operatorname{DepLit}(\phi)\) removes
consequences involving irrelevant literals. Since the atom set is finite,
\(\operatorname{CN}_{\mathcal F}(\phi)\) is finite.

For example,
\[
\operatorname{CN}_{\mathcal F}(p)=\{p\},
\qquad
\operatorname{CN}_{\mathcal F}(p\lor q)=\{p\lor q\},
\]
and
\[
\operatorname{CN}_{\mathcal F}(p\land q)
=
\{p,q,p\lor q,p\land q\}.
\]
The last equality holds because \(p\land q\) entails \(p\), entails \(q\),
entails \(p\lor q\), and entails itself. These are the distinct representative
consequences over the literals on which \(p\land q\) depends.

Having introduced logical arguments, we can now define similarity measures over them. The logical representation lets us compare arguments along two dimensions: their premises and their conclusions.

\begin{definition}[Similarity measure]
A similarity measure on a sentential logic \((\mathcal L,\vdash)\) is a map
\[
\mathcal S:\operatorname{Arg}(\mathcal L)\times
\operatorname{Arg}(\mathcal L)\to[0,1].
\]
\end{definition}

\citet{David2021} studies several similarity measures for logical arguments and their properties. The central idea is that arguments can be compared syntactically, by the overlap in their premises, and semantically, by the overlap in the logical consequences of their conclusions. Measures based on only one of these dimensions fail to satisfy all the desired properties. The syntactic--semantic Jaccard measure combines both dimensions.

\begin{definition}[Syntactic--semantic Jaccard similarity]
Let \(a=(\Phi,\phi)\) and \(b=(\Psi,\psi)\) be arguments in
\(\operatorname{Arg}(\mathcal L)\). For \(0<\sigma<1\), define
\[
\operatorname{sim}^{\sigma}(a,b)
=
\sigma s_{\mathrm{syn}}(\Phi,\Psi)
+
(1-\sigma)s_{\mathrm{sem}}(\phi,\psi),
\]
where the premise similarity is
\[
s_{\mathrm{syn}}(\Phi,\Psi)
=
\begin{cases}
\dfrac{
|\operatorname{Co}(\Phi,\Psi)|
}{
|\Phi|+|\Psi|-|\operatorname{Co}(\Phi,\Psi)|
}
& \text{if } \Phi\neq\emptyset \text{ and } \Psi\neq\emptyset,\\[2mm]
1 & \text{if } \Phi=\Psi=\emptyset,\\
0 & \text{otherwise,}
\end{cases}
\]
and the conclusion similarity is
\[
s_{\mathrm{sem}}(\phi,\psi)
=
\frac{
|\operatorname{CN}_{\mathcal F}(\phi)
\cap
\operatorname{CN}_{\mathcal F}(\psi)|
}{
|\operatorname{CN}_{\mathcal F}(\phi)
\cup
\operatorname{CN}_{\mathcal F}(\psi)|
}.
\]
\end{definition}

The measure is a convex combination of premise similarity and conclusion similarity. The term \(s_{\mathrm{syn}}\) measures how many premises the two arguments share, modulo logical equivalence. The term \(s_{\mathrm{sem}}\) measures how many canonical consequences their conclusions share. The parameter \(\sigma\) controls the relative weight placed on premises and conclusions.

Importantly, the value of \(s_{\mathrm{sem}}\) is invariant to the particular choice of
\(\mathcal F\), as long as \(\mathcal F\) is a valid finite CNF representative
language. Different valid choices merely select different formulas from the same
logical-equivalence classes. They therefore relabel the elements of
\(\operatorname{CN}_{\mathcal F}(\phi)\) without changing the cardinalities of
the intersections and unions defining the quantity.

\subsection{Example of computation}

We illustrate the computation with two simple arguments.

\begin{quote}
\textbf{Argument 1.} John is Susan's brother. Susan is shorter than all her
brothers. Susan is the only girl in the household. Therefore, Susan is the
shortest sibling.
\end{quote}

\begin{quote}
\textbf{Argument 2.} Susan is John's sister. Susan is shorter than all her
brothers. No boy in the household has blond hair. Susan is the only girl in the
household. There is at least one sibling with blond hair. Therefore, Susan is
the shortest sibling and is the only one with blond hair.
\end{quote}

Let \(a_1=(\Phi_1,\phi_1)\) and \(a_2=(\Phi_2,\phi_2)\) denote the corresponding
logical arguments. We encode the premises as follows:

\[
\begin{array}{lll}
p_1 & = & \text{``John is Susan's brother,''} \\
p_2 & = & \text{``Susan is shorter than all her brothers,''} \\
p_3 & = & \text{``Susan is the only girl in the household,''}
\end{array}
\qquad
\Phi_1=\{p_1,p_2,p_3\}.
\]

For the second argument, write

\[
\begin{array}{lll}
q_1 & = & \text{``Susan is John's sister,''} \\
q_2 & = & \text{``Susan is shorter than all her brothers,''} \\
q_3 & = & \text{``No boy in the household has blond hair,''} \\
q_4 & = & \text{``Susan is the only girl in the household,''} \\
q_5 & = & \text{``There is at least one sibling with blond hair,''}
\end{array}
\qquad
\Phi_2=\{q_1,q_2,q_3,q_4,q_5\}.
\]

The shared premises, modulo logical equivalence, are

\[
p_1\equiv q_1,
\qquad
p_2=q_2,
\qquad
p_3=q_4.
\]

Thus,

\[
\operatorname{Co}(\Phi_1,\Phi_2)=\{p_1,p_2,p_3\}.
\]

The syntactic similarity is therefore

\[
s_{\mathrm{syn}}(\Phi_1,\Phi_2)
=
\frac{
|\operatorname{Co}(\Phi_1,\Phi_2)|
}{
|\Phi_1|+|\Phi_2|-|\operatorname{Co}(\Phi_1,\Phi_2)|
}
=
\frac{3}{3+5-3}
=
\frac{3}{5}.
\]

Now let

\[
r_1=\text{``Susan is the shortest sibling''}
\qquad\text{and}\qquad
r_2=\text{``Susan is the only one with blond hair.''}
\]

Then

\[
\phi_1=r_1
\qquad\text{and}\qquad
\phi_2=r_1\land r_2.
\]

Under the fixed CNF representative convention,

\[
\operatorname{CN}_{\mathcal F}(\phi_1)=\{r_1\},
\]

whereas

\[
\operatorname{CN}_{\mathcal F}(\phi_2)
=
\{r_1,r_2,r_1\lor r_2,r_1\land r_2\}.
\]

The two conclusion consequence sets share only \(r_1\). Hence,

\[
s_{\mathrm{sem}}(\phi_1,\phi_2)
=
\frac{
|\operatorname{CN}_{\mathcal F}(\phi_1)
\cap
\operatorname{CN}_{\mathcal F}(\phi_2)|
}{
|\operatorname{CN}_{\mathcal F}(\phi_1)
\cup
\operatorname{CN}_{\mathcal F}(\phi_2)|
}
=
\frac{1}{4}.
\]

Combining the two terms gives, for \(0<\sigma<1\),

\[
\operatorname{sim}^{\sigma}(a_1,a_2)
=
\sigma\left(\frac{3}{5}\right)
+
(1-\sigma)\left(\frac{1}{4}\right)
=
\frac{1}{4}+\frac{7}{20}\sigma.
\]

\subsection{Formal definitions of the properties satisfied by similarity measures} \label{A: SimNotions}

We state the relevant properties for an arbitrary similarity measure
\(S:\operatorname{Arg}(\mathcal L)\times\operatorname{Arg}(\mathcal L)\to[0,1]\).
\citet{David2021} proves that the syntactic--semantic Jaccard similarity satisfies each property below.

\begin{property}[Maximality]
A similarity measure \(S\) satisfies maximality if, for every
\(a\in\operatorname{Arg}(\mathcal L)\),
\[
S(a,a)=1.
\]
\end{property}

Maximality says that each argument is maximally similar to itself. Since
similarity values lie in \([0,1]\), maximal similarity is represented by the
value \(1\).

\begin{property}[Symmetry]
A similarity measure \(S\) satisfies symmetry if, for all
\(a,b\in\operatorname{Arg}(\mathcal L)\),
\[
S(a,b)=S(b,a).
\]
\end{property}

Symmetry says that similarity does not depend on the order in which the two
arguments are compared.

\begin{property}[Triangle inequality]
A similarity measure \(S\) satisfies the triangle inequality if, for all
\(a,b,c\in\operatorname{Arg}(\mathcal L)\),
\[
1+S(a,c)\geq S(a,b)+S(b,c).
\]
\end{property}

This condition is the triangle inequality written in similarity form. Equivalently,
the dissimilarity \(d(a,b)=1-S(a,b)\) satisfies
\[
d(a,c)\leq d(a,b)+d(b,c).
\]
Thus, if \(a\) is close to \(b\) and \(b\) is close to \(c\), then \(a\) cannot
be arbitrarily far from \(c\).

\begin{property}[Substitution]
A similarity measure \(S\) satisfies substitution if, for all
\(a,b,c\in\operatorname{Arg}(\mathcal L)\),
\[
S(a,b)=1
\quad\Longrightarrow\quad
S(a,c)=S(b,c).
\]
\end{property}

Substitution says that maximally similar arguments are interchangeable for
similarity comparisons with any third argument.\\

For a formula \(\phi\), let \(\operatorname{At}(\phi)\) denote the set of atoms
occurring in \(\phi\). For a finite set of formulas \(\Phi\), define
\[
\operatorname{At}(\Phi)
=
\bigcup_{\phi\in\Phi}\operatorname{At}(\phi).
\]

\begin{property}[Minimality]
A similarity measure \(S\) satisfies minimality if, for all
\(a=(\Phi,\phi)\) and \(b=(\Psi,\psi)\) in
\(\operatorname{Arg}(\mathcal L)\), the following conditions imply
\(S(a,b)=0\):
\begin{enumerate}
    \item \(a\not\approx b\);
    \item \(\operatorname{At}(\Phi)\cap\operatorname{At}(\Psi)=\emptyset\);
    \item \(\operatorname{At}(\phi)\cap\operatorname{At}(\psi)=\emptyset\).
\end{enumerate}
\end{property}

Minimality says that arguments with no shared content in both their premises and their conclusions must have zero similarity.

\begin{property}[Non-zero]
A similarity measure \(S\) satisfies non-zero similarity if, for all
\(a=(\Phi,\phi)\) and \(b=(\Psi,\psi)\) in
\(\operatorname{Arg}(\mathcal L)\),
\[
\operatorname{Co}(\Phi,\Psi)\neq\emptyset
\quad\Longrightarrow\quad
S(a,b)>0.
\]
\end{property}

The non-zero property says that shared premise content is enough to induce
positive similarity.

\begin{property}[Strict monotony]
A similarity measure \(S\) satisfies monotony if, for all
\(a=(\Phi,\phi)\), \(b=(\Psi,\psi)\), and \(c=(\Xi,\xi)\) in
\(\operatorname{Arg}(\mathcal L)\), the following conditions imply
\(S(a,b)\geq S(a,c)\):
\begin{enumerate}
    \item \(\phi\equiv\psi\) or
    \(\operatorname{At}(\phi)\cap\operatorname{At}(\xi)=\emptyset\);
    \item \(\operatorname{Co}(\Phi,\Xi)\subseteq
    \operatorname{Co}(\Phi,\Psi)\);
    \item
    \[
    \Psi\setminus\operatorname{Co}(\Psi,\Phi)
    =
    \operatorname{Co}\bigl(
    \Psi\setminus\operatorname{Co}(\Psi,\Phi),
    \Xi\setminus\operatorname{Co}(\Xi,\Phi)
    \bigr).
    \]
\end{enumerate}
It satisfies strict monotony if, in addition, either
\[
\operatorname{Co}(\Phi,\Xi)\subsetneq
\operatorname{Co}(\Phi,\Psi),
\]
or
\[
\operatorname{Co}(\Phi,\Xi)\neq\emptyset
\quad\text{and}\quad
|\Xi\setminus\operatorname{Co}(\Xi,\Phi)|
>
|\Psi\setminus\operatorname{Co}(\Psi,\Phi)|.
\]
In either strict case,
\[
S(a,b)>S(a,c).
\]
\end{property}

Strict monotony says that, holding the conclusion comparison fixed or irrelevant,
an argument becomes more similar to \(a\) as it shares more of \(a\)'s premises
and adds fewer unrelated premises.

\begin{property}[Strict dominance]
A similarity measure \(S\) satisfies dominance if, for all
\(a=(\Phi,\phi)\), \(b=(\Psi,\psi)\), and \(c=(\Xi,\xi)\) in
\(\operatorname{Arg}(\mathcal L)\), the following conditions imply
\(S(a,b)\geq S(a,c)\):
\begin{enumerate}
    \item \(\Psi\cong\Xi\);
    \item
    \[
    \operatorname{CN}_{\mathcal F}(\phi)
    \cap
    \operatorname{CN}_{\mathcal F}(\xi)
    \subseteq
    \operatorname{CN}_{\mathcal F}(\phi)
    \cap
    \operatorname{CN}_{\mathcal F}(\psi);
    \]
    \item
    \[
    \operatorname{CN}_{\mathcal F}(\psi)
    \setminus
    \operatorname{CN}_{\mathcal F}(\phi)
    \subseteq
    \operatorname{CN}_{\mathcal F}(\xi)
    \setminus
    \operatorname{CN}_{\mathcal F}(\phi).
    \]
\end{enumerate}
It satisfies strict dominance if, in addition, either the inclusion in condition
2 is strict, or
\[
\operatorname{CN}_{\mathcal F}(\phi)
\cap
\operatorname{CN}_{\mathcal F}(\xi)
\neq\emptyset
\]
and the inclusion in condition 3 is strict. In either strict case,
\[
S(a,b)>S(a,c).
\]
\end{property}

Strict dominance says that, when two candidate arguments have equivalent premise
sets, the one whose conclusion shares more finite logical consequences with the
reference conclusion is more similar.

\subsection{More information on Reproducing Kernel Hilbert Spaces} \label{A:rkhs}

Let $X$ be a set. A function
\[
K:X\times X\to\mathbb R
\]
is a positive semi-definite kernel if, for every finite collection
$x_1,\ldots,x_n\in X$, the Gram matrix
\[
\mathbf K
=
\big(K(x_i,x_j)\big)_{i,j=1}^n
\]
is positive semi-definite. Equivalently, for every
$c_1,\ldots,c_n\in\mathbb R$,
\[
\sum_{i=1}^n\sum_{j=1}^n
c_i c_j K(x_i,x_j)
\geq 0.
\]

A reproducing kernel Hilbert space (RKHS) associated with $K$ is a Hilbert space
$\mathcal H_K$ of real-valued functions on $X$ such that, for every $x\in X$,
the function
\[
K(\cdot,x):X\to\mathbb R
\]
belongs to $\mathcal H_K$, and the reproducing property holds:
\[
h(x)
=
\langle h,K(\cdot,x)\rangle_{\mathcal H_K}
\qquad
\text{for all } h\in\mathcal H_K.
\]
The Moore--Aronszajn theorem states that every positive semi-definite kernel
$K$ uniquely determines such an RKHS $\mathcal H_K$
\citep{Aronszajn1950,Paulsen_Raghupathi_2016}.

The associated canonical feature map is
\[
\Theta_K:X\to\mathcal H_K,
\qquad
\Theta_K(x)=K(\cdot,x).
\]
Thus, each object $x\in X$ is represented by the function that records its
kernel similarity to every other element of $X$. In this sense, $\Theta_K$ is an embedding: it represents each element of $X$ as a
point in a Hilbert space whose inner products are given by $K$. The key identity is
\[
\langle
\Theta_K(x),
\Theta_K(y)
\rangle_{\mathcal H_K}
=
K(x,y),
\]
which follows directly from the reproducing property. Hence the kernel can be
interpreted as an inner product between the feature representations
$\Theta_K(x)$ and $\Theta_K(y)$, even when these representations are not written
explicitly or are infinite-dimensional.\\

\subsection{Finite-dimensional approximation by kernel PCA}
\label{A:kernel_pca}

The logical embedding defined in Section~\ref{S:logical_embeddings} maps each
argument \(a\) to an element
\[
\Theta_\sigma(a)\in\mathcal H_\sigma
\]
of the RKHS induced by the kernel \(\operatorname{sim}^{\sigma}\). This space
may be infinite-dimensional, so the embedding is not directly a finite vector
representation.

The canonical RKHS feature map has the form
\[
\Theta_\sigma(a)(\cdot)=\operatorname{sim}^{\sigma}(\cdot,a).
\]
Thus, the embedding of \(a\) is a function that records the similarity of \(a\)
to every possible argument. On a finite corpus \(a_1,\ldots,a_n\), we observe
only the restriction of this function to the sample:
\[
\bigl(
\operatorname{sim}^{\sigma}(a_1,a),
\ldots,
\operatorname{sim}^{\sigma}(a_n,a)
\bigr)=(\Theta_\sigma(a)(a_1), \cdots,\Theta_\sigma(a)(a_n)).
\]

By restricting to the argument corpus, these restricted feature functions form
the columns of the Gram matrix
\[
K^\sigma_{ij}
=
\operatorname{sim}^{\sigma}(a_i,a_j).
\]

\noindent where each column $j$ is then a finite-dimensional approximation to the logical embedding of argument $j$.  When \(n\) is large, however, representing each argument by its similarities to all observed arguments is usually too high-dimensional. Kernel PCA compresses this information by projecting the centered
RKHS embeddings onto their leading empirical principal directions. The resulting
coordinates give the best rank-\(r\) linear approximation to the centered
logical embeddings on the observed corpus.

In practice, we do not observe \(K^\sigma\) exactly. We estimate it from
natural-language text, obtaining \(\widehat K^\sigma\). Kernel PCA applied to
\(\widehat K^\sigma\) therefore provides a finite-dimensional approximation to
the ideal logical embeddings, with two sources of approximation: the empirical
similarity estimates and the rank-\(r\) PCA truncation.

We now state the finite-sample construction for this approximation. Let
\(K:X\times X\to\mathbb R\) be a positive semidefinite kernel on a set \(X\),
with RKHS \(\mathcal H_K\), and let \(x_1,\ldots,x_n\in X\). Define the Gram
matrix
\[
\mathbf K=(K(x_i,x_j))_{i,j=1}^n.
\]
Let
\[
\mathbf H=I_n-\frac{1}{n}\mathbf 1\mathbf 1^\top
\qquad\text{and}\qquad
\mathbf K_c=\mathbf H\mathbf K\mathbf H
\]
be the centering matrix and the centered Gram matrix.

Let
\[
\Theta_K(x)=K(\cdot,x)
\]
be the canonical RKHS feature map, and define the centered embeddings
\[
Z_i
=
\Theta_K(x_i)-\frac{1}{n}\sum_{j=1}^n\Theta_K(x_j).
\]
The centered Gram matrix records exactly the inner products among these centered
embeddings:
\[
(\mathbf K_c)_{ij}
=
\langle Z_i,Z_j\rangle_{\mathcal H_K}.
\]
Thus, kernel PCA performs ordinary PCA on the centered RKHS embeddings
\(Z_1,\ldots,Z_n\), using only the finite matrix \(\mathbf K_c\).

Let
\[
\mathbf K_c=U\Lambda U^\top
\]
be an eigendecomposition, with eigenvalues
\[
\lambda_1\geq\lambda_2\geq\cdots\geq\lambda_n\geq0
\]
and orthonormal eigenvectors \(u_1,\ldots,u_n\). For each
\(\lambda_\ell>0\), the \(\ell\)th empirical principal direction in
\(\mathcal H_K\) is
\[
\hat v_\ell
=
\frac{1}{\sqrt{\lambda_\ell}}
\sum_{j=1}^n (u_\ell)_j Z_j .
\]
The coordinate of \(x_i\) along this direction is
\[
\langle Z_i,\hat v_\ell\rangle_{\mathcal H_K}
=
\sqrt{\lambda_\ell}(u_\ell)_i.
\]
Therefore, the rank-\(r\) kernel PCA coordinates are given by the rows of
\[
U_r\Lambda_r^{1/2},
\]
where \(U_r\) contains the first \(r\) eigenvectors and \(\Lambda_r\) contains
the corresponding eigenvalues.

These coordinates have the standard PCA optimality property. Among all
rank-\(r\) orthogonal projections \(P\) on \(\mathcal H_K\), the projection onto
\[
\operatorname{span}\{\hat v_1,\ldots,\hat v_r\}
\]
minimizes the average squared reconstruction error:
\[
\min_{\operatorname{rank}(P)\leq r}
\frac{1}{n}\sum_{i=1}^n
\|Z_i-PZ_i\|_{\mathcal H_K}^2
=
\sum_{\ell>r}\frac{\lambda_\ell}{n}.
\]
Equivalently, by the Eckart--Young--Mirsky theorem, the rank-\(r\) truncation
\[
U_r\Lambda_r U_r^\top
\]
is the best rank-\(r\) approximation to \(\mathbf K_c\) in Frobenius norm:
\[
\min_{\operatorname{rank}(M)\leq r}
\|\mathbf K_c-M\|_F^2
=
\sum_{\ell>r}\lambda_\ell^2.
\]

Applying this construction with
\(X=\operatorname{Arg}(\mathcal L)\) and
\(K=\operatorname{sim}^{\sigma}\) shows that kernel PCA on the ideal Gram matrix
\(K^\sigma\) gives the best \(r\)-dimensional linear approximation to the
centered logical embeddings of the observed arguments. In practice, we use
\(\widehat K^\sigma\) in place of \(K^\sigma\). When
\(\widehat K^\sigma\) is a good approximation to \(K^\sigma\), the resulting
coordinates approximate the finite-sample kernel PCA coordinates of the ideal
logical embeddings. Thus, the empirical coordinates used in the paper are
low-dimensional approximations to the logical embeddings, restricted to the
observed corpus.

The reduction of kernel PCA to the Gram matrix eigenproblem follows
\citet{Scholkopf1998}; see also \citet{ScholkopfSmola2002}. The best
rank-\(r\) matrix approximation property follows from
\citet{EckartYoung1936}.

\subsection{Omitted proofs}

\subsubsection{Auxiliary Lemmas}

Given a set $X$ and an equivalence relation $\sim$ on $X$, the quotient of $X$
by $\sim$ is
\[
\faktor{X}{\sim}:=\{[x]_\sim: x\in X\},
\]
where $[x]_\sim=\{y\in X:y\sim x\}$. For $\Phi\subseteq\mathcal L$, write
\[
\faktor{\Phi}{\equiv}:=\{[\xi]_{\equiv}:\xi\in\Phi\},
\]
where $[\xi]_{\equiv}=\{\eta\in\mathcal L:\eta\equiv\xi\}$ is the equivalence
class of $\xi$ in the ambient language $\mathcal L$.

\begin{lemma}\label{AL:NoEquivalentPremises}
Let $a=(\Phi,\phi)\in\operatorname{Arg}(\mathcal L)$. If
$\xi,\zeta\in\Phi$ and $\xi\equiv\zeta$, then $\xi=\zeta$.
\end{lemma}

\begin{proof}
Suppose, toward a contradiction, that $\xi,\zeta\in\Phi$, $\xi\neq\zeta$, and
$\xi\equiv\zeta$. Let $\Phi'=\Phi\setminus\{\xi\}$. We show that
$\Phi'\vdash\phi$.

Let $v$ be any valuation that satisfies every formula in $\Phi'$. Since
$\zeta\in\Phi'$, $v$ satisfies $\zeta$. Because $\zeta\equiv\xi$, every
valuation satisfying $\zeta$ also satisfies $\xi$. Hence $v$ satisfies $\xi$.
Therefore $v$ satisfies every formula in $\Phi$. Since $\Phi\vdash\phi$, it
follows that $v$ satisfies $\phi$.

Thus every valuation satisfying $\Phi'$ also satisfies $\phi$, so
$\Phi'\vdash\phi$. But $\Phi'\subsetneq\Phi$, contradicting the
non-redundancy of $\Phi$. Therefore $\xi=\zeta$.
\end{proof}

\begin{lemma}\label{AL: Lemma1Kernel}
Let $a=(\Phi,\phi)$ and $b=(\Psi,\psi)$ be arguments in
$\operatorname{Arg}(\mathcal L)$. Then the map
\[
\pi:\operatorname{Co}(\Phi,\Psi)
\to
\faktor{\Phi}{\equiv}\cap\faktor{\Psi}{\equiv},
\qquad
\xi\mapsto[\xi]_{\equiv},
\]
is a bijection. In particular,
\[
|\operatorname{Co}(\Phi,\Psi)|
=
\left|
\faktor{\Phi}{\equiv}\cap\faktor{\Psi}{\equiv}
\right|.
\]
\end{lemma}

\begin{proof}
First, the map is well-defined. If $\xi\in\operatorname{Co}(\Phi,\Psi)$, then
$\xi\in\Phi$ and there exists $\zeta\in\Psi$ such that $\xi\equiv\zeta$. Hence
$[\xi]_{\equiv}\in\faktor{\Phi}{\equiv}$ and
$[\xi]_{\equiv}=[\zeta]_{\equiv}\in\faktor{\Psi}{\equiv}$. Therefore
$[\xi]_{\equiv}\in\faktor{\Phi}{\equiv}\cap\faktor{\Psi}{\equiv}$.

We now prove surjectivity. Let
\[
C\in\faktor{\Phi}{\equiv}\cap\faktor{\Psi}{\equiv}.
\]
Since $C\in\faktor{\Phi}{\equiv}$, there exists $\xi\in\Phi$ such that
$C=[\xi]_{\equiv}$. Since $C\in\faktor{\Psi}{\equiv}$, there exists
$\zeta\in\Psi$ such that $C=[\zeta]_{\equiv}$. Hence $\xi\equiv\zeta$, so
$\xi\in\operatorname{Co}(\Phi,\Psi)$. Moreover, $\pi(\xi)=[\xi]_{\equiv}=C$.
Thus $\pi$ is surjective.

We now prove injectivity. Let $\xi,\xi'\in\operatorname{Co}(\Phi,\Psi)$ and
suppose that $\pi(\xi)=\pi(\xi')$. Then
\[
[\xi]_{\equiv}=[\xi']_{\equiv},
\]
so $\xi\equiv\xi'$. Since $\xi,\xi'\in\Phi$ and $\Phi$ is the premise set of an
argument, Lemma~\ref{AL:NoEquivalentPremises} implies that $\xi=\xi'$.
Therefore, $\pi$ is injective.

Hence $\pi$ is bijective.
\end{proof}

\begin{lemma}\label{AL: Lemma2Kernel}
Let $a=(\Phi,\phi)$ be an argument in $\operatorname{Arg}(\mathcal L)$.
Then
\[
|\Phi|=\left|\faktor{\Phi}{\equiv}\right|.
\]
\end{lemma}

\begin{proof}
Consider the map
\[
\pi:\Phi\to\faktor{\Phi}{\equiv},
\qquad
\xi\mapsto[\xi]_{\equiv}.
\]
This map is surjective by definition of $\faktor{\Phi}{\equiv}$.\\

We now show that it is injective. Let $\xi,\zeta\in\Phi$ and suppose that
$\pi(\xi)=\pi(\zeta)$. Then
\[
[\xi]_{\equiv}=[\zeta]_{\equiv},
\]
so $\xi\equiv\zeta$. Since $\Phi$ is the premise set of an argument,
Lemma~\ref{AL:NoEquivalentPremises} implies that $\xi=\zeta$. Hence $\pi$ is
injective. Therefore $\pi$ is a bijection, and consequently
\[
|\Phi|=\left|\faktor{\Phi}{\equiv}\right|.
\]
\end{proof}

\subsubsection{Proof of Theorem \ref{T:Kernel}} \label{A: Kernel}

\begin{proof}
Let $a=(\Phi,\phi)$ and $b=(\Psi,\psi)$ be arguments in
$\operatorname{Arg}(\mathcal L)$.

For any set $X$, write
\[
\mathcal P_{\mathrm{fin}}(X)
=
\{A\subseteq X: |A|<\infty\}
\]
for the set of finite subsets of $X$.

First define the usual Tanimoto coefficient on pairs of finite sets with
nonempty union by
\[
\tilde{k}_T(A,B)
=
\frac{|A\cap B|}{|A\cup B|},
\qquad A\cup B\neq\varnothing .
\]
This is the Tanimoto kernel on finite sets, also known as the Jaccard kernel
\citep[p.~301]{Passerini2013}. To handle the only case in which this expression
is undefined, we extend it by setting
\[
k_T(A,B)
=
\begin{cases}
\tilde{k}_T(A,B), & \text{if } A\cup B\neq\varnothing,\\[1.2em]
1, & \text{if } A=B=\varnothing .
\end{cases}
\]

We first verify that this extension is still positive semi-definite. Let
$A_1,\ldots,A_n$ be finite sets. After permutation, we may assume without loss of generality that
\[
A_1=\cdots=A_m=\varnothing
\]
and that $A_{m+1},\ldots,A_n$ are nonempty. If $m=0$, all sets are nonempty, and
the result follows directly from the positive semi-definiteness of the usual
Tanimoto kernel. If $m=n$, then the Gram matrix is the all-ones matrix, which is
positive semi-definite. Thus assume $0<m<n$.

For $i\leq m<j$, we have
\[
k_T(A_i,A_j)
=
\tilde{k}_T(\varnothing,A_j)
=
\frac{|\varnothing\cap A_j|}{|\varnothing\cup A_j|}
=
0.
\]
Therefore the Gram matrix $\big(k_T(A_i,A_j)\big)_{i,j=1}^n$ has the block form
\[
\begin{pmatrix}
\mathbf 1_m\mathbf 1_m^\top & 0 \\
0 & G
\end{pmatrix},
\]
where $\mathbf 1_m\mathbf 1_m^\top$ is the $m\times m$ all-ones matrix and
\[
G
=
\big(\tilde{k}_T(A_i,A_j)\big)_{i,j=m+1}^n
\]
is the Gram matrix of the usual Tanimoto kernel on the nonempty sets
$A_{m+1},\ldots,A_n$.

The first block is positive semi-definite because, for any
$u\in\mathbb R^m$,
\[
u^\top(\mathbf 1_m\mathbf 1_m^\top)u
=
(\mathbf 1_m^\top u)^2
\geq 0.
\]
The second block is positive semi-definite by the positive semi-definiteness of
the Tanimoto kernel \citep[p.~301]{Passerini2013}. Hence the block-diagonal
Gram matrix is positive semi-definite. Therefore $k_T$ is a positive
semi-definite kernel on $\mathcal P_{\mathrm{fin}}(X)$, for any ambient set
$X$.

We now rewrite the syntactic part of the similarity. By
Lemmas~\ref{AL: Lemma1Kernel} and~\ref{AL: Lemma2Kernel},
\[
|\operatorname{Co}(\Phi,\Psi)|
=
\left|
\faktor{\Phi}{\equiv}\cap\faktor{\Psi}{\equiv}
\right|,
\]
and
\[
|\Phi|
=
\left|\faktor{\Phi}{\equiv}\right|,
\qquad
|\Psi|
=
\left|\faktor{\Psi}{\equiv}\right|.
\]
Therefore, whenever
$\faktor{\Phi}{\equiv}\cup\faktor{\Psi}{\equiv}\neq\varnothing$, inclusion--
exclusion gives
\[
s_{\mathrm{syn}}(\Phi,\Psi)
=
\frac{
\left|\faktor{\Phi}{\equiv}\cap\faktor{\Psi}{\equiv}\right|
}{
\left|\faktor{\Phi}{\equiv}\cup\faktor{\Psi}{\equiv}\right|
}
=
k_T\left(
\faktor{\Phi}{\equiv},
\faktor{\Psi}{\equiv}
\right).
\]
If
$\faktor{\Phi}{\equiv}=\faktor{\Psi}{\equiv}=\varnothing$, then
$\Phi=\Psi=\varnothing$, and both sides are equal to $1$ by the empty-set
convention for $k_T$. Hence, in all cases,
\[
s_{\mathrm{syn}}(\Phi,\Psi)
=
k_T\left(
\faktor{\Phi}{\equiv},
\faktor{\Psi}{\equiv}
\right).
\]

Similarly, by definition of the semantic component,
\[
s_{\mathrm{sem}}(\phi,\psi)
=
k_T\left(
\operatorname{CN}_{\mathcal F}(\phi),
\operatorname{CN}_{\mathcal F}(\psi)
\right),
\]
again using the convention that the similarity of two empty consequence sets is
$1$.

Now define
\[
\Pi:\operatorname{Arg}(\mathcal L)\to
\mathcal P_{\mathrm{fin}}\left(\faktor{\mathcal L}{\equiv}\right),
\qquad
\Pi(\Phi,\phi)=\faktor{\Phi}{\equiv},
\]
and
\[
\rho:\operatorname{Arg}(\mathcal L)\to
\mathcal P_{\mathrm{fin}}(\mathcal F),
\qquad
\rho(\Phi,\phi)=\operatorname{CN}_{\mathcal F}(\phi).
\]
These maps are well-defined because premise sets are finite and
$\mathcal F$ is finite.

Define
\[
K_{\Pi}(a,b)=k_T(\Pi(a),\Pi(b))
\]
and
\[
K_{\rho}(a,b)=k_T(\rho(a),\rho(b)).
\]
Since $k_T$ is positive semi-definite on finite subsets, $K_{\Pi}$ and
$K_{\rho}$ are positive semi-definite kernels on
$\operatorname{Arg}(\mathcal L)$ by closure of positive semi-definite kernels
under pullback along arbitrary maps
\citep{Aronszajn1950,Paulsen_Raghupathi_2016}.

Finally, for any $0<\sigma<1$,
\[
\operatorname{sim}^{\sigma}(a,b)
=
\sigma K_{\Pi}(a,b)
+
(1-\sigma)K_{\rho}(a,b).
\]
Positive semi-definite kernels are closed under nonnegative scalar
multiplication and finite sums
\citep{Aronszajn1950,Paulsen_Raghupathi_2016}. Since $\sigma>0$ and
$1-\sigma>0$, it follows that $\operatorname{sim}^{\sigma}$ is a positive
semi-definite kernel on $\operatorname{Arg}(\mathcal L)$.
\end{proof}

\subsubsection{Proof of Theorem \ref{T: CKernel}} \label{A: CKernel}

\begin{proof}
Let $a,b\in\operatorname{Arg}(\mathcal L)$. Since
$\Theta_{\sigma}$ is the canonical feature map associated with
$\operatorname{sim}^{\sigma}$, we have
\[
\langle
\Theta_{\sigma}(x),
\Theta_{\sigma}(y)
\rangle_{\mathcal H_{\sigma}}
=
\operatorname{sim}^{\sigma}(x,y)
\]
for all $x,y\in\operatorname{Arg}(\mathcal L)$.

First suppose that $\Theta_{\sigma}(a)=\Theta_{\sigma}(b)$. Then
\[
0
=
\|\Theta_{\sigma}(a)-\Theta_{\sigma}(b)\|_{\mathcal H_{\sigma}}^2.
\]
Expanding the squared norm gives
\[
0
=
\operatorname{sim}^{\sigma}(a,a)
+
\operatorname{sim}^{\sigma}(b,b)
-
2\operatorname{sim}^{\sigma}(a,b).
\]
Since each argument has maximal similarity with itself,
\[
\operatorname{sim}^{\sigma}(a,a)
=
\operatorname{sim}^{\sigma}(b,b)
=
1.
\]
Therefore
\[
\operatorname{sim}^{\sigma}(a,b)=1.
\]
By Theorem~\ref{T: Metrics}, this implies $a\approx b$.

Conversely, suppose that $a\approx b$. By Theorem~\ref{T: Metrics},
\[
\operatorname{sim}^{\sigma}(a,b)=1.
\]
Again using
$\operatorname{sim}^{\sigma}(a,a)=\operatorname{sim}^{\sigma}(b,b)=1$, we get
\[
\|\Theta_{\sigma}(a)-\Theta_{\sigma}(b)\|_{\mathcal H_{\sigma}}^2
=
1+1-2
=
0.
\]
Hence $\Theta_{\sigma}(a)=\Theta_{\sigma}(b)$.

Thus
\[
\Theta_{\sigma}(a)=\Theta_{\sigma}(b)
\quad\text{if and only if}\quad
a\approx b.
\]

It remains only to check that the quotient map is well-defined. If
$[a]_{\approx}=[b]_{\approx}$, then $a\approx b$, and the result just proved
implies
\[
\Theta_{\sigma}(a)=\Theta_{\sigma}(b).
\]
Therefore
\[
\overline{\Theta}_{\sigma}([a]_{\approx})=\Theta_{\sigma}(a)
\]
is well-defined. Moreover, if
\[
\overline{\Theta}_{\sigma}([a]_{\approx})
=
\overline{\Theta}_{\sigma}([b]_{\approx}),
\]
then $\Theta_{\sigma}(a)=\Theta_{\sigma}(b)$, so $a\approx b$, and hence
\[
[a]_{\approx}=[b]_{\approx}.
\]
Therefore $\overline{\Theta}_{\sigma}$ is injective.
\end{proof}

\section{Exhibits}
\subsection{Features of the Logical Embedding}
In this section, we present a series of exhibits that illustrate more clearly the functioning and advantages of the Logical Embeddings method. In the introduction we provided an example of how regular methods may fail to recover logical similarity.

A useful comparison is to compare the Gram matrix generated by the Semantic-Syntactic Jaccard similarity metric to a Gram matrix generated by LLM based comparisons. Our prompt is in Appendix~\ref{app:llm-prompts}. When we generate this matrix we find that LLMs rarely assign zero logical similarity. Instead, it almost always detects some vague theme or rhetorical overlap. The Logical Embeddings method is significantly stricter. It relies on clause-level decomposition and mutual entailment and assigns a positive similarity only when there is a clear logical relation between the components of two arguments. The result is a much sparser network, but where the connections are on average stronger and more interpretable. 

Figure~\ref{fig:network_doping_comparison} illustrates this point using only the subset of arguments under the doping topic from from the IBM-ArgQ-6.3kArgs dataset \citep{toledo2019}. Although the two graphs contain the same number of nodes, corresponding to the same set of arguments, their edge structures are radically different. The full-argument graph is nearly complete, with $89598$ edges, covering $98.5\%$ of all possible pairs. The Logical Embeddings graph, in contrast, contains only $6408$ edges, or $7.0\%$ of all possible pairs. It reveals that the full-argument approach struggles to separate arguments based on their relation, and produce an almost fully connected graph where there is no possibility to see meaningful distinctions. 

This has direct implications for classification. In the full-argument graph, pro and con arguments are mixed, which suggests that the similarity measure does not recover the latent argumentative structure of the topic. The Logical Embeddings graph instead displays a much clearer separation between the two sides. This is because it is a measure designed to capture logical rather than rhetorical or thematic proximity. Being sparse it is not an issue, it represents the selectivity of the model to connect two arguments.

\begin{figure}[t]
\centering
\caption{Logical Embeddings and Full-argument LLM.}
\includegraphics[width=\columnwidth]{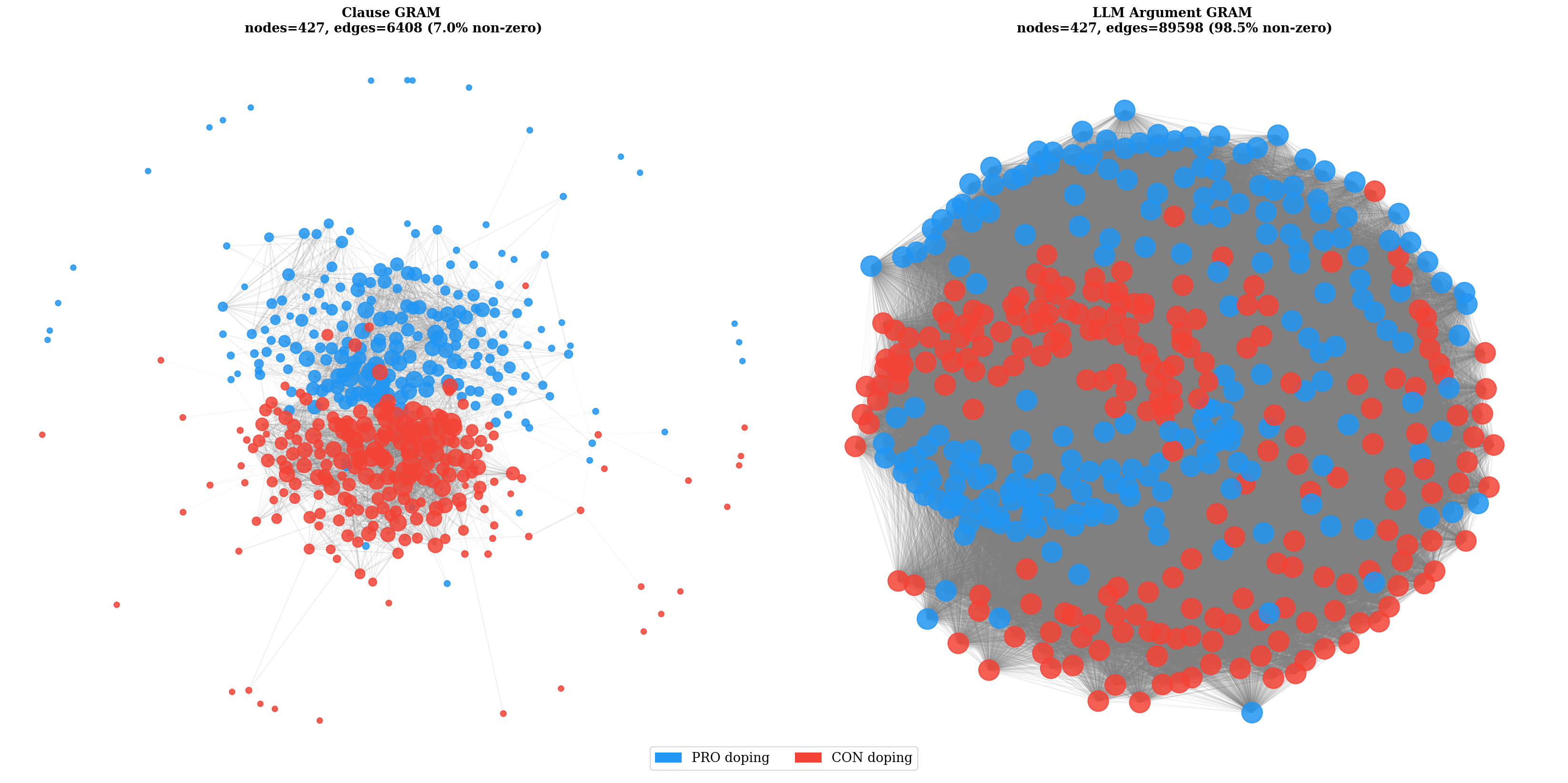}

\captionsetup{font={scriptsize}, width={0.5\textwidth}, justification=justified}
\caption*{\textit{Notes}: \networkcomparison }
\label{fig:network_doping_comparison}
\end{figure}

This leads to a final advantage of the Logical Embeddings model: traceability. Since the score is produced from the number of clauses covered by mutual entailment, it is possible to trace back exactly what is connected and why the LLM judged those clauses as entailed. Figure~\ref{fig:gram_configurations} illustrates schematically how the score is obtained in the simplest, and also most recurrent, case in our dataset: two premises and one subconclusion. First of all, the two families of clauses composing the text, premises and subconclusions, are treated separately and given equal weight. Each component is obtained by dividing the number of shared clauses by the total number of distinct clauses in that family. In our example, if the only subconclusion is entailed, then it contributes $0.5*1 = 0.5$ to the final score. The figure also makes clear another main feature of the model: it is not the raw number of mutual entailments that determines the score, but rather the number of clauses covered by those entailments. Looking at the scores $0.25$ and $0.5$ obtained using premises only, the number of mutual entailments is the same, but in the second case all clauses are connected, which gives a score of $0.5 \times1/1$, while in the first case the score is only $0.5 \times1/2$. This has clear implications, especially for higher scores, and helps explain why Logical Embeddings can display a higher average score than full-argument similarity. While mutual entailment and clause level comparison impose a stricter rule for connection, once a pair is judged entailed it has greater contributive power.
\begin{figure}[t]
\centering
\caption{How Logical Embeddings scores are computed?}
\includegraphics[width=0.85\columnwidth]{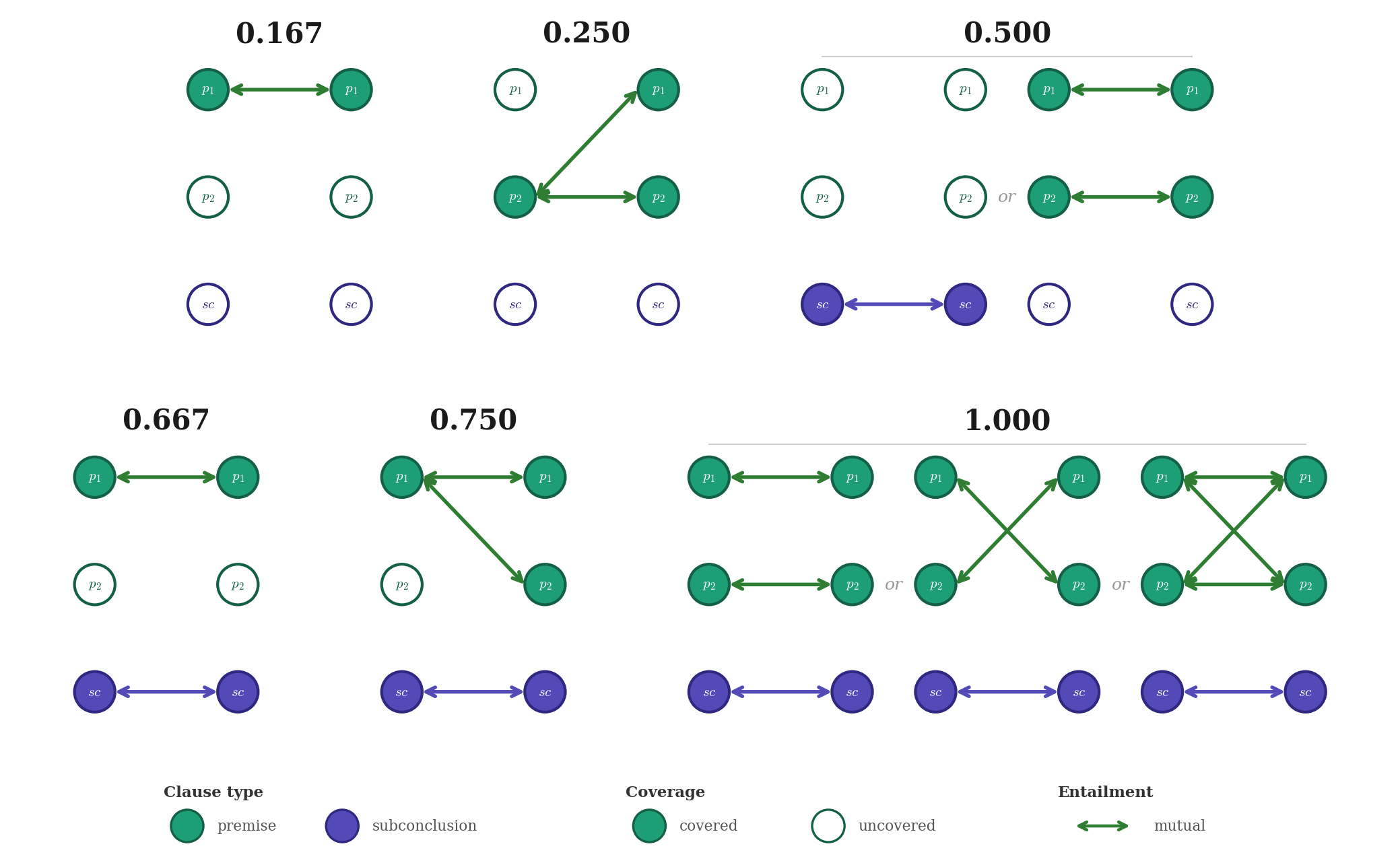}
\label{fig:gram_configurations}
\captionsetup{font={scriptsize}, width={0.5\textwidth}, justification=justified}
\caption*{\textit{Notes}: \gramconfiguration }
\end{figure}

To further illustrate the mechanics of the method, Figure~\ref{fig:exhibit_gram_77_vs_4612} shows exactly how a similarity score of $0.75$ is obtained. First, both texts are decomposed into premises and subconclusions. First of all the subconclusions are judged mutually entailed. Not making a vaccine mandatory implies that it is illegitimate to impose vaccination on individuals who oppose it for religious reasons and viceversa. This already contributes $0.5$ to the final score.

Moving to the premises, the second premise of the first argument, \textit{forcing vaccination violates personal convictions}, is judged to be entailed by both premises of the second argument, namely that mandatory vaccination infringes personal rights and that respecting bodily autonomy is fundamental for personal freedom. By contrast, the first premise of the first argument, \textit{some people have strong religious beliefs against vaccines}, is not entailed by either premise of the second argument, because it does not by itself imply bodily autonomy, nor does it imply a right to reject a vaccine mandate. We therefore end up with two distinct premises overall, of which one is effectively shared across the two arguments through mutual entailment. What matters here is that entailment creates a bridge between clauses: they are treated as similar when one logically entails the other and viceversa. In this example, this means that one premise is shared, while the other remains unmatched. The premise component is therefore equal to $\tfrac{1}{2}$, that contributes $0.25$ weighted. Adding this to the $0.5$ of the subconclusion, the overall similarity score is $0.75$.

\begin{figure}[t]
    \centering
    \caption{Construction of a Logical Embedding}
        \label{fig:exhibit_gram_77_vs_4612}
    \includegraphics[width=0.6\columnwidth]{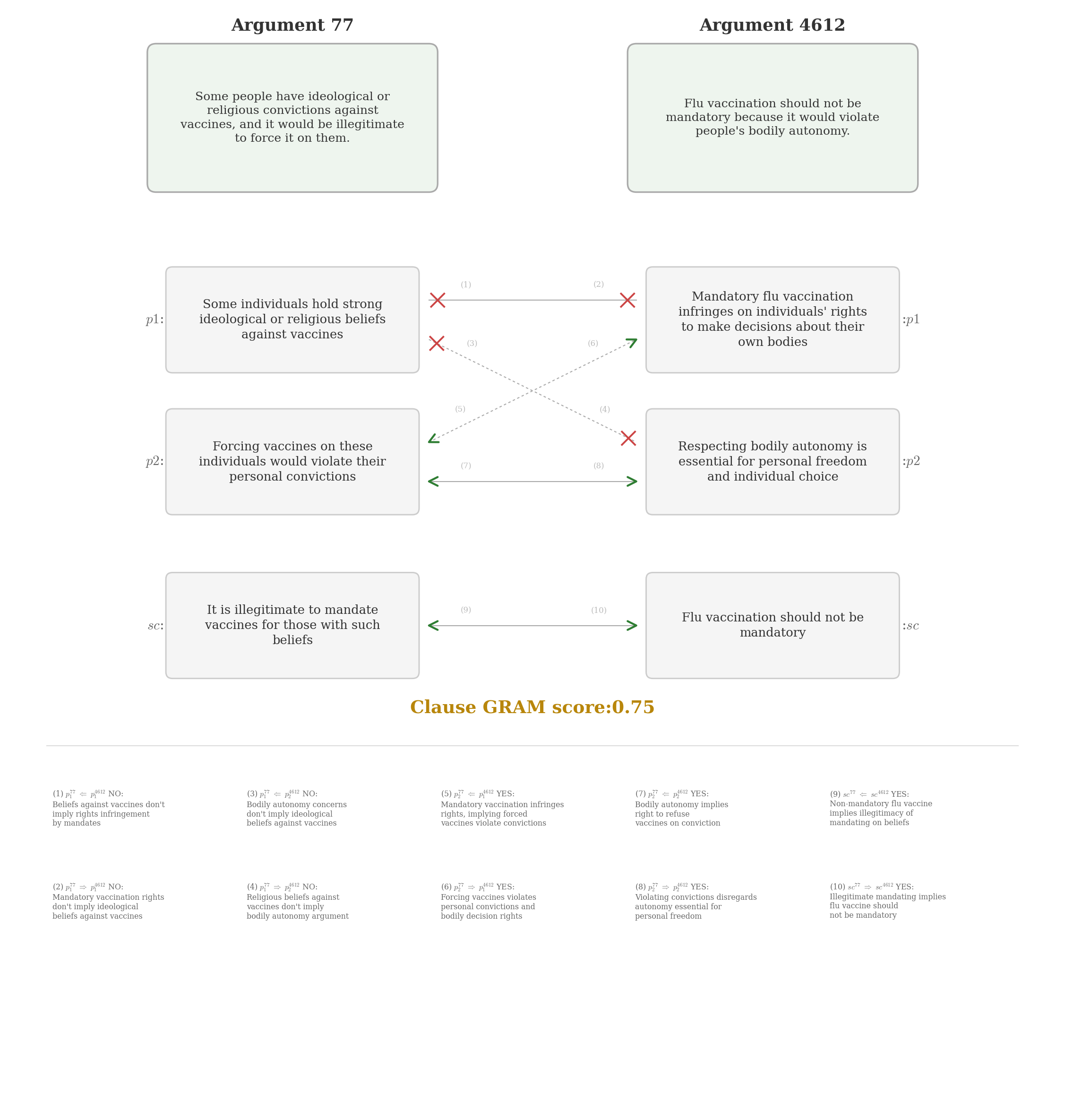}
    \begin{minipage}{\columnwidth}
    \footnotesize
    \textit{Notes:} \exhibitspecific
    \end{minipage}
\end{figure}

\subsection{F1 Scores}

In the main body of the paper we reported F1-scores for the performance of logical embeddings on a standard classification task. Here we provide additional results. 
The logical embeddings used in the evaluation in the paper retain the top 100 eigenvectors of Gram matrix. When all eigenvectors are retained, the neural network model used as one of the evaluation models suffers from severe overfitting due to the high effective dimensionality of the embedding space. Figure~\ref{fig:eigenvalue_sweep_f1} illustrates this: test F1 decreases as more eigenvectors are included, while the train-test gap widens. 100 eigenvectors are chosen as the value that maximizes the F1.

\begin{figure}[t]
    \centering
    \caption{Neural network F1 score by number of eigenvectors}
    \label{fig:eigenvalue_sweep_f1}
    \includegraphics[width=0.75\columnwidth]{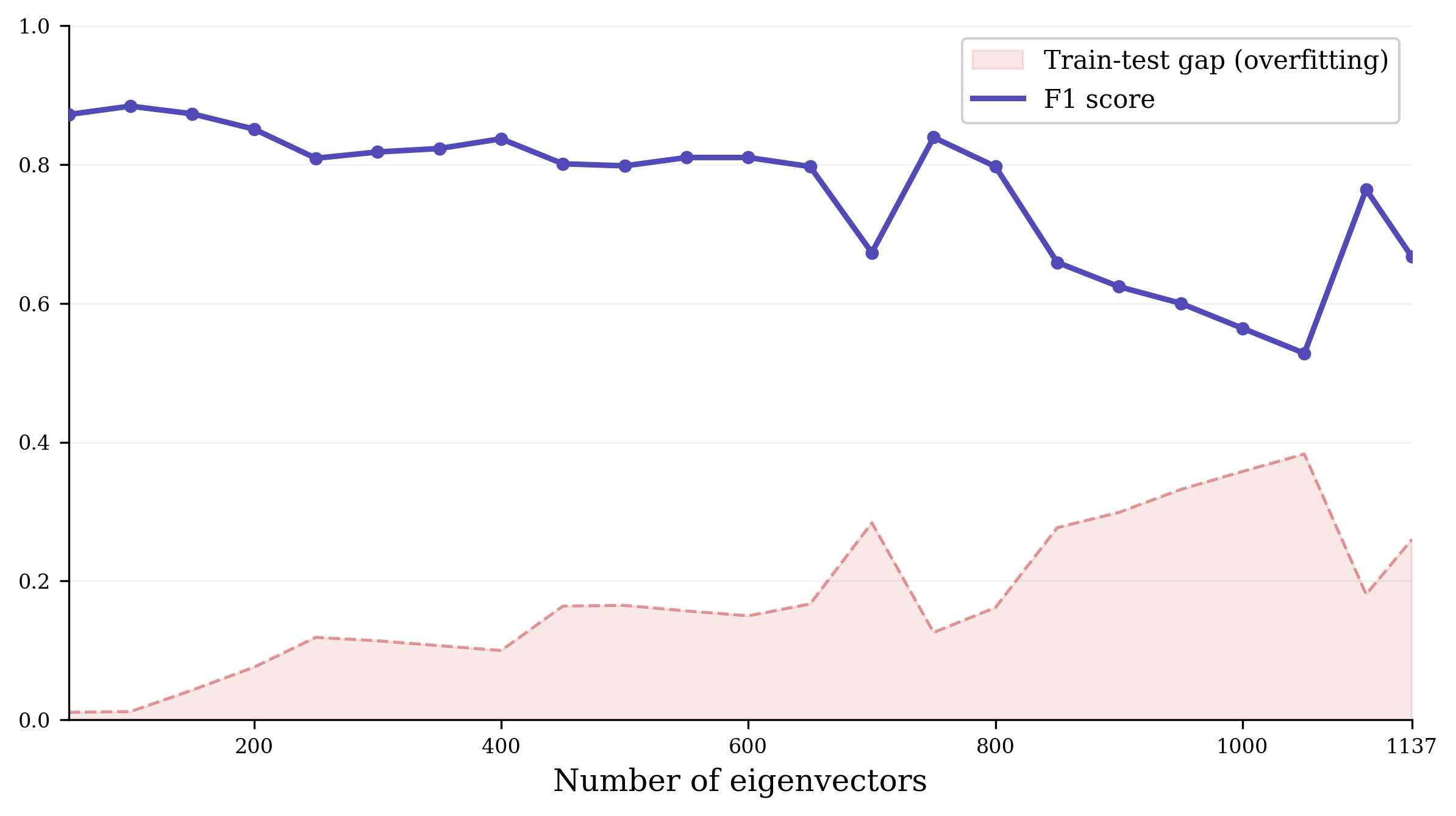}
    \begin{minipage}{\columnwidth}
    \footnotesize
    \textit{Notes:} \eigenvaluesweep
    \end{minipage}
\end{figure}

\subsection{Practical implementation}

In the paper, Theorem~3 establishes that
\[
\operatorname{sim}^{\sigma}(a,b)
=
\sigma \, k_T\bigl(\Pi(a),\Pi(b)\bigr)
+
(1-\sigma)\, k_T\bigl(\rho(a),\rho(b)\bigr)
\]
is a positive semi-definite kernel on $\operatorname{Arg}(\mathcal L)$, provided that $\Pi$ and $\rho$ are \emph{fixed} maps from arguments to reduced representations. The essential requirement is that the same reduction rule be applied to every argument independently of the particular pair under comparison.

In the main empirical analysis, however, the implementation is pairwise. For a given pair of arguments $(a,b)$, equivalence classes are constructed using only the clauses appearing in those two arguments. Thus, instead of a single global reduction map $\Pi(a)=\faktor{\Phi}{\equiv}$, the implementation uses a pair dependent map
\[
\Pi_{a,b}(a)=\faktor{\Phi}{\equiv_{a,b}},
\]
. Different entries of the Gram matrix are computed under different equivalence relations. The fixed-map argument underlying Theorem~3 therefore no longer applies, and positive semi-definiteness is not guaranteed.

However the pairwise estimator remains substantively the best way to compute the true similarity scores. It is closer to the way a human evaluator would compare arguments, because it focuses only on the local logical relation between the two arguments at hand and it is also easier to trace since when two arguments are matched, one can directly inspect which premises and subconclusions were identified as equivalent within that particular comparison. Empirically, this pairwise implementation yields the strongest performance.

The main difficulty in constructing a PSD alternative is therefore not simply restoring a fixed map, but to do so without generating an excessively general global transitivity. A naive global equivalence relation based on connected components in the bidirectional entailment graph creates a small number of extremely large components. Due to the large number of arguments, and for LLM fluctuation in evaluations in these clusters, clauses may be reduced not because they are directly close in meaning, but because they are connected through long chains of bidirectional entailment. The practical problem is therefore to reduce these giant global components into smaller and more coherent equivalence classes while keeping the map fixed (and so all the properties).

This global robustness implementation addresses this issue through a structural graph partition procedure. First, we build the clause-level graph induced by bidirectional entailment relations. Each node is a clause, and an edge is present whenever two clauses are judged to be in bidirectional entailment. We then compute the connected components of this graph. Any component whose size already falls below a fixed threshold $M$ is left unchanged (that we put to be  $M=20$).

However the big majority of connections falls into 4 big clusters that exceeds this threshold and where we don't collapse it directly into one equivalence class. Instead, we recursively partition it by a balanced spectral split. To do so, we assign each edge a structural weight
\[
w(i,j)=1+\lvert \Gamma(i)\cap \Gamma(j)\rvert,
\]
where $\Gamma(i)$ denotes the neighborhood of clause $i$ in the entailment graph. Hence, edges whose endpoints share many neighbors are treated as stronger and more internally supported, while edges with little local support are easier to cut. On the weighted subgraph induced by the oversized component, we use the Fiedler vector of the Laplacian matrix to produce a bipartition. The split is applied recursively on the median, chosen to be the one that creates the more balanced partition, were all components have the same $M=20$ size.

Only after this recursive reduction do we collapse clauses into equivalence classes. That is, each final component of size at most $M$ is treated as one global equivalence class and represented by a single node. This yields fixed global reduction maps $\Pi^{G}$ and $\rho^{G}$ prior to any pairwise comparison. The resulting similarity can therefore be written as
\[
K^{G}_{\sigma}(a,b)
=
\sigma \, k_T\bigl(\Pi^{G}(a),\Pi^{G}(b)\bigr)
+
(1-\sigma)\, k_T\bigl(\rho^{G}(a),\rho^{G}(b)\bigr).
\]
Since $\Pi^{G}$ and $\rho^{G}$ are defined once and for all, independently of the particular pair $(a,b)$, Theorem~3 applies directly and the corresponding Gram matrix is positive semi-definite by construction.

While so this sacrifices local flexibility, since every clause must ultimately belong to a single global equivalence class, it imposes a much more disciplined reduction, because giant transitive components are recursively broken into bounded-size subcomponents by a structural partition rule rather than being collapsed wholesale.

\begin{table}[h!]
\footnotesize
\centering
\captionsetup{justification=centering}
\caption{\textsc{F1 scores (100 eigenvectors)}}\label{TABLE:F1 results_2}
\resizebox{0.75\columnwidth}{!}{%
\begin{tabular}{lccccc}
\toprule\toprule
\emph{Estimation:} & \multicolumn{5}{c}{F1 Scores}\\\cmidrule(lr){2-6}
& \multicolumn{1}{c}{Logistic} & \multicolumn{1}{c}{Lasso} & \multicolumn{1}{c}{Ridge} &  \multicolumn{1}{c}{Neural Network} & \multicolumn{1}{c}{Random Forest} \\
\midrule\\
\input{Figures/table_f1_comparison_global}
\end{tabular}
}
\captionsetup{font={scriptsize}, width={\columnwidth}, justification=justified}
\caption*{\textit{Notes}: \mainresults}
\end{table}

\subsubsection{LLM prompts}
\label{app:llm-prompts}

This section reports the prompts used for the LLM-based components of the analysis. We use two distinct prompts. The first prompt is a direct full-argument similarity score between two arguments. This prompt is used to construct the LLM-based comparison Gram matrix and to generate the Figure \ref{fig:exhibit_figure}. The second prompt is the main one used in the construction of Logical Embeddings. It asks the model to evaluate directed entailment between clause-level statements. Logical equivalence is then defined by mutual entailment.

\paragraph{Full-argument similarity prompt.}

The following prompt is used to obtain direct LLM-based argument similarity scores. The model receives two complete arguments and returns a continuous score between 0 and 1, together with a short explanation.

\begin{lstlisting}[basicstyle=\ttfamily\small, breaklines=true]
You are assessing the degree to which two arguments share the same underlying logical content.

Assign a continuous score from 0 to 1 reflecting how much the logical content of these two arguments overlap, where 0 means the arguments share no logical content whatsoever and 1 means the arguments are logically equivalent. Keep it within two decimal places.

Important: focus on the logical structure and inferential content, not the topic or surface wording.

Consider the following two arguments:

Argument 1: {text1}
Argument 2: {text2}

Please provide your response as a valid JSON object:
{
  "sentence_id_1": "{sentence_id_1}",
  "sentence_id_2": "{sentence_id_2}",
  "answer": "N/A",
  "score": <a float between 0 and 1>,
  "reasoning": "brief reasoning for your score",
  "comment": "N/A"
}

Ensure the response is strictly a valid JSON object with no extra characters or formatting.
\end{lstlisting}

\paragraph{Clause-level entailment prompt.}

The following prompt is used to evaluate directed entailment between clause-level statements. For each ordered pair of clauses, the model returns a YES/NO entailment judgment, a short reasoning, and a confidence score. We run the prompt in both directions. Two clauses are treated as logically equivalent only when both directed entailment judgments are positive.

\begin{lstlisting}[basicstyle=\ttfamily\small, breaklines=true]
You are working on assessing whether statements made in certain contexts entail one another. In natural language, logical implication makes sense in context. For example, if a statement speaks about a ruler and another one speaks about a king, they may be talking about the same figure of authority, if both statements talk about or within a monarchy.

Question: Do the ideas entailed in Statement 1 imply the ideas entailed in Statement 2? Answer with YES or NO, state a brief reasoning for your answer, then give a score from 1 to 10 evaluating how confident you are with your answer (1 means 'I am absolutely not confident about my answer' and 10 means 'I am completely sure about my assessment'), and give a brief comment on your level of certainty on your answer.

Please provide your response as a valid JSON object in the following format:

Please consider the following statements:

Statement 1: {text1}
Statement 2: {text2}

{
  "sentence_id_1": "{sentence_id_1}",
  "sentence_id_2": "{sentence_id_2}",
  "answer": "Your YES or NO answer goes here (as a string)",
  "reasoning": "The reasoning behind your answer goes here (as a string)",
  "score": "A score between 0 and 10 based on your confidence (as an integer)",
  "comment": "Additional comments on your confidence here (as a string)"
}

Ensure the response is strictly a valid JSON object with no extra characters or formatting.
\end{lstlisting}

\end{document}

%% file: notes.tex
\newcommand{\mainresults}{This Table presents the main F1 scores of the paper. All classifiers are estimated on a 70/30 train-test split of 1137 arguments. Logistic, Lasso and Ridge are linear models with L1, L1 and L2 penalties respectively. Random Forest uses 200 trees. The Neural Network is a single hidden layer network whose hyperparameters are selected via 5-fold validation grid search with early stopping. Logical emb. refers to the clause-level logical embedding using the top 100 eigenvectors of the pairwise entailment score matrix. GloVe uses 100-dimensional averaged word vectors. BERT and RoBERTa use 768-dimensional CLS token embeddings. SBERT uses 384-dimensional sentence embeddings. GPT uses OpenAI text-embedding-3-small (1536 dimensions). All embeddings are standardised before estimation. Bold indicates the highest F1 per column. }

\newcommand{\exhibit}{The figure shows how full-text similarity scores could not correspond to argument equivalence. The left argument receives a high score because it shares the same abstract concept of necessity despite concerning a different domain. The right argument instead have a lower score despite topic overlap because it is relative to a more specific context. This motivates distinguishing between surface level semantic similarity and argument level equivalence. }

\newcommand{\exhibitspecific}{This figure illustrates how the Clause GRAM score is constructed for a pair of arguments. Arguments 77 and 4612 are decomposed into clause-level components (two premises and one conclusion). Green \textgreater denote positive entailment judgments, while red \texttimes denote negative ones. In this example, the first premise of argument 77 is not bidirectionally entailed by either premise of argument 4612, whereas the second premise is bidirectionally entailed with both premises of argument 4612, producing a score of 0.5. The two conclusions are instead bidirectionally entailed, yielding a conclusion level score of 1. With equal weight assigned to premise and conclusion components, the resulting Clause GRAM score is 0.75. The bottom panel reports the LLM reasoning used to determine whether each pair of clauses is entailed or not. }

\newcommand{\networkcomparison}{The Figure reports, for the subset of arguments on doping, the network of pairwise connections generated by each method. Each node represents an argument and each edge represents a non-zero connection between two arguments. Blue nodes denote pro-doping arguments, while red nodes denote con-doping arguments. }

\newcommand{\gramconfiguration}{The Figure shows how the Logical Embeddings score is computed for each possible entailment configuration in the case of an argument composed of two premises and one subconclusion. Green nodes denote premises and purple nodes denote subconclusions. A full node indicates a covered clause, a clause for which at least one mutual entailment is identified. Only mutual entailments are represented by arrows. }

\newcommand{\eigenvaluesweep}{The Figure show the behaviour of the Neural Network model under changes of eigenvector retained. The solid line reports the test F1 score of the neural network across different numbers of eigenvectors. The shaded area represents the gap in F1 between the training and the test scores. }

%% file: Figures/table_f1_comparison_pairwise.tex
Logical emb. & 0.842 & 0.883 & \textbf{0.869} & \textbf{0.884} & 0.857 \\
GloVe & 0.668 & 0.700 & 0.668 & 0.667 & 0.667 \\
BERT & 0.684 & 0.775 & 0.665 & 0.766 & 0.694 \\
SBERT & 0.674 & 0.761 & 0.670 & 0.781 & 0.760 \\
RoBERTa & 0.725 & 0.696 & 0.688 & 0.751 & 0.733 \\
GPT & 0.826 & 0.837 & 0.780 & 0.872 & 0.840 \\
\midrule
GloVe + Logical emb. & 0.766 & 0.814 & 0.793 & 0.834 & 0.871 \\
BERT + Logical emb. & 0.844 & 0.838 & 0.693 & 0.803 & 0.866 \\
SBERT + Logical emb. & 0.787 & 0.845 & 0.757 & 0.860 & 0.871 \\
RoBERTa + Logical emb. & 0.836 & 0.868 & 0.752 & 0.825 & 0.870 \\
GPT + Logical emb. & \textbf{0.871} & \textbf{0.885} & 0.781 & 0.875 & \textbf{0.890} \\

%% file: Figures/table_f1_comparison_global.tex
Logical emb. & 0.848 & 0.851 & \textbf{0.853} & 0.840 & 0.829 \\
GloVe & 0.668 & 0.700 & 0.668 & 0.667 & 0.667 \\
BERT & 0.667 & 0.775 & 0.665 & 0.766 & 0.694 \\
SBERT & 0.674 & 0.761 & 0.670 & 0.781 & 0.760 \\
RoBERTa & 0.723 & 0.696 & 0.688 & 0.751 & 0.733 \\
GPT & 0.832 & 0.837 & 0.780 & 0.871 & 0.843 \\
\midrule
GloVe + Logical emb. & 0.761 & 0.820 & 0.819 & 0.828 & 0.837 \\
BERT + Logical emb. & 0.837 & 0.833 & 0.748 & 0.824 & 0.799 \\
SBERT + Logical emb. & 0.768 & 0.855 & 0.771 & 0.768 & 0.833 \\
RoBERTa + Logical emb. & 0.838 & 0.851 & 0.728 & 0.806 & 0.811 \\
GPT + Logical emb. & \textbf{0.864} & \textbf{0.880} & 0.807 & \textbf{0.879} & \textbf{0.880} \\

%% file: FinalVersion.bib
@inproceedings{Fukumizu2008,
 author = {Fukumizu, Kenji and Gretton, Arthur and Sch\"{o}lkopf, Bernhard and Sriperumbudur, Bharath K.},
 booktitle = {Advances in Neural Information Processing Systems},
 editor = {D. Koller and D. Schuurmans and Y. Bengio and L. Bottou},
 pages = {},
 publisher = {Curran Associates, Inc.},
 title = {Characteristic Kernels on Groups and Semigroups},
 volume = {21},
 year = {2008}
}

@article{fukumizu2004dimensionality,
  title={Dimensionality reduction for supervised learning with reproducing kernel Hilbert spaces},
  author={Fukumizu, Kenji and Bach, Francis R and Jordan, Michael I},
  journal={Journal of Machine Learning Research},
  volume={5},
  number={Jan},
  pages={73--99},
  year={2004}
}

@Inbook{Passerini2013,
author="Passerini, Andrea",
editor="Bianchini, Monica
and Maggini, Marco
and Jain, Lakhmi C.",
title="Kernel Methods for Structured Data",
bookTitle="Handbook on Neural Information Processing",
year="2013",
publisher="Springer Berlin Heidelberg",
address="Berlin, Heidelberg",
pages="283--333",
isbn="978-3-642-36657-4",
doi="10.1007/978-3-642-36657-4_9"
}

@article{Aronszajn1950,
 ISSN = {00029947},
 author = {N. Aronszajn},
 journal = {Transactions of the American Mathematical Society},
 number = {3},
 pages = {337--404},
 publisher = {American Mathematical Society},
 title = {Theory of Reproducing Kernels},
 volume = {68},
 year = {1950}
}

@book{Paulsen_Raghupathi_2016,
place={Cambridge},
series={Cambridge Studies in Advanced Mathematics},
title={An Introduction to the Theory of Reproducing Kernel Hilbert Spaces},
publisher={Cambridge University Press},
author={Paulsen, Vern I. and Raghupathi, Mrinal},
year={2016},
collection={Cambridge Studies in Advanced Mathematics}
}

@inproceedings{reimers2019,
    title = "Classification and Clustering of Arguments with Contextualized Word Embeddings",
    author = "Reimers, Nils  and
      Schiller, Benjamin  and
      Beck, Tilman  and
      Daxenberger, Johannes  and
      Stab, Christian  and
      Gurevych, Iryna",
    editor = "Korhonen, Anna  and
      Traum, David  and
      M{\`a}rquez, Llu{\'\i}s",
    booktitle = "Proceedings of the 57th Annual Meeting of the Association for Computational Linguistics",
    year = "2019",
    address = "Florence, Italy",
    publisher = "Association for Computational Linguistics",
    doi = "10.18653/v1/P19-1054",
    pages = "567--578"
}

@inproceedings{toledo2019,
    title = "Automatic Argument Quality Assessment - New Datasets and Methods",
    author = "Toledo, Assaf  and
      Gretz, Shai  and
      Cohen-Karlik, Edo  and
      Friedman, Roni  and
      Venezian, Elad  and
      Lahav, Dan  and
      Jacovi, Michal  and
      Aharonov, Ranit  and
      Slonim, Noam",
    editor = "Inui, Kentaro  and
      Jiang, Jing  and
      Ng, Vincent  and
      Wan, Xiaojun",
    booktitle = "Proceedings of the 2019 Conference on Empirical Methods in Natural Language Processing and the 9th International Joint Conference on Natural Language Processing (EMNLP-IJCNLP)",
    year = "2019",
    address = "Hong Kong, China",
    publisher = "Association for Computational Linguistics",
    doi = "10.18653/v1/D19-1564",
    pages = "5625--5635"
}

@article{Stab2017,
    author = {Stab, Christian and Gurevych, Iryna},
    title = "{Parsing Argumentation Structures in Persuasive Essays}",
    journal = {Computational Linguistics},
    volume = {43},
    number = {3},
    pages = {619-659},
    year = {2017},
    issn = {0891-2017},
    doi = {10.1162/COLI_a_00295}
}

@book{govier2013practical,
  title={A practical study of argument},
  author={Govier, Trudy},
  year={2013},
  publisher={Cengage Learning}
}

@inproceedings{peldszus-stede-2015-joint,
    title = "Joint prediction in {MST}-style discourse parsing for argumentation mining",
    author = "Peldszus, Andreas  and
      Stede, Manfred",
    editor = "M{\`a}rquez, Llu{\'\i}s  and
      Callison-Burch, Chris  and
      Su, Jian",
    booktitle = "Proceedings of the 2015 Conference on Empirical Methods in Natural Language Processing",
    month = sep,
    year = "2015",
    address = "Lisbon, Portugal",
    publisher = "Association for Computational Linguistics",
    doi = "10.18653/v1/D15-1110",
    pages = "938--948",
}

@article{Palau2011,
author = {Mochales-Palau, Raquel  and Moens, Marie-Francine},
title = {Argumentation mining},
year = {2011},
issue_date = {March 2011},
publisher = {Kluwer Academic Publishers},
address = {USA},
volume = {19},
number = {1},
issn = {0924-8463},
doi = {10.1007/s10506-010-9104-x},
journal = {Artif. Intell. Law},
pages = {1–22},
numpages = {22}
}

@misc{Liu2023,
      title={Unsupervised Argument Similarity via Sentence Compression}, 
      author={Siyi Liu},
      year={2023},
      eprint={2302.12490},
      archivePrefix={arXiv},
}

@inproceedings{Amgoud2018,
  TITLE = {{Measuring Similarity between Logical Arguments}},
  AUTHOR = {Amgoud, Leila and David, Victor},
  URL = {https://hal.science/hal-02325829},
  BOOKTITLE = {{16th International Conference on Principles of Knowledge Representation and Reasoning (KR 2018)}},
  ADDRESS = {Tempe, United States},
  PAGES = {1-21},
  YEAR = {2018}
}

@phdthesis{David2021,
  TITLE = {{Dealing with Similarity in Argumentation}},
  AUTHOR = {David, Victor},
  NUMBER = {2021TOU30055},
  SCHOOL = {{Universit{\'e} Paul Sabatier - Toulouse III}},
  YEAR = {2021}
}

@article{simpson2013mathematical,
  title={Mathematical logic},
  author={Simpson, Stephen G},
  journal={Lecture Notes for Introductory Courses in Mathematical Logic. The Pennsylvania State University, University Park, State College},
  year={2013}
}

@article{Besnard2001,
title = {A logic-based theory of deductive arguments},
journal = {Artificial Intelligence},
volume = {128},
number = {1},
pages = {203-235},
year = {2001},
issn = {0004-3702},
doi = {https://doi.org/10.1016/S0004-3702(01)00071-6},
author = {Philippe Besnard and Anthony Hunter}
}

@article{russell1995modern,
  title={Artificial Intelligence: A modern approach},
  author={Russell, Stuart and Norvig, Peter},
  journal={Artificial Intelligence. Prentice-Hall, Englewood Cliffs},
  volume={25},
  number={27},
  pages={79--80},
  year={1995}
}

@inproceedings{amgoud2021compilation,
  title={Similarity Measures Based on Compiled Arguments},
  author={Amgoud, Leila and David, Victor},
  booktitle={Symbolic and Quantitative Approaches to Reasoning with Uncertainty: 16th European Conference, ECSQARU 2021, Prague, Czech Republic, September 21--24, 2021, Proceedings 16},
  pages={32--44},
  year={2021},
  organization={Springer}
}

@article{jaccard1908nouvelles,
  title={Nouvelles recherches sur la distribution florale},
  author={Jaccard, Paul},
  journal={Bull. Soc. Vaud. Sci. Nat.},
  volume={44},
  pages={223--270},
  year={1908}
}

@inproceedings{Steck2024,
  title={Is Cosine-Similarity of Embeddings Really About Similarity?},
  author={Harald Steck and Chaitanya Ekanadham and Nathan Kallus},
  year={2024}
}

@inproceedings{maccartney-manning-2008-modeling,
    title = "Modeling Semantic Containment and Exclusion in Natural Language Inference",
    author = "MacCartney, Bill  and
      Manning, Christopher D.",
    editor = "Scott, Donia  and
      Uszkoreit, Hans",
    booktitle = "Proceedings of the 22nd International Conference on Computational Linguistics (Coling 2008)",
    year = "2008",
    address = "Manchester, UK",
    publisher = "Coling 2008 Organizing Committee",
    pages = "521--528",
}

@misc{wang2021entailmentfewshotlearner,
      title={Entailment as Few-Shot Learner}, 
      author={Sinong Wang and Han Fang and Madian Khabsa and Hanzi Mao and Hao Ma},
      year={2021},
      eprint={2104.14690},
      archivePrefix={arXiv}
}

@misc{pilault2022conditionallyadaptivemultitasklearning,
      title={Conditionally Adaptive Multi-Task Learning: Improving Transfer Learning in NLP Using Fewer Parameters \& Less Data}, 
      author={Jonathan Pilault and Amine Elhattami and Christopher Pal},
      year={2022},
      eprint={2009.09139},
      archivePrefix={arXiv}
}

@misc{zhang2020semanticsawarebertlanguageunderstanding,
      title={Semantics-aware BERT for Language Understanding}, 
      author={Zhuosheng Zhang and Yuwei Wu and Hai Zhao and Zuchao Li and Shuailiang Zhang and Xi Zhou and Xiang Zhou},
      year={2020},
      eprint={1909.02209},
      archivePrefix={arXiv},
      primaryClass={cs.CL}
}

@article{vaswani2017attention,
  title={Attention is all you need},
  author={Vaswani, Ashish and Shazeer, Noam and Parmar, Niki and Uszkoreit, Jakob and Jones, Llion and Gomez, Aidan N and Kaiser, {\L}ukasz and Polosukhin, Illia},
  journal={Advances in neural information processing systems},
  volume={30},
  year={2017}
}

@book{Shawe-Taylor2004,
  title={Kernel Methods for Pattern Analysis},
  author={Shawe-Taylor, John and Cristianini, Nello},
  year={2004},
  publisher={Cambridge University Press},
  address={Cambridge},
}

@article{EckartYoung1936,
  author    = {Carl Eckart and Gale Young},
  title     = {The approximation of one matrix by another of lower rank},
  journal   = {Psychometrika},
  volume    = {1},
  number    = {3},
  pages     = {211--218},
  year      = {1936},
  doi       = {10.1007/BF02288367}
}

@article{Scholkopf1998,
  author    = {Bernhard Sch\"{o}lkopf and Alexander Smola and Klaus-Robert M\"{u}ller},
  title     = {Nonlinear component analysis as a kernel eigenvalue problem},
  journal   = {Neural Computation},
  volume    = {10},
  number    = {5},
  pages     = {1299--1319},
  year      = {1998},
  doi       = {10.1162/089976698300017467}
}

@book{ScholkopfSmola2002,
  author    = {Bernhard Sch\"{o}lkopf and Alexander J. Smola},
  title     = {Learning with Kernels: Support Vector Machines, Regularization,
               Optimization, and Beyond},
  publisher = {MIT Press},
  address   = {Cambridge, MA},
  year      = {2002},
  isbn      = {978-0262194754}
}
